\documentclass[conference]{IEEEtran}
\IEEEoverridecommandlockouts

\usepackage{booktabs, tabularx}
\usepackage{multirow}
\usepackage{tikz}
\usepackage{amsmath}
\usepackage[binary-units]{siunitx}
\usepackage[abbreviations]{foreign}
\usepackage{xspace}
\usepackage[inline]{enumitem}
\usepackage{url}
\usepackage{hyperref}
\usepackage[most]{tcolorbox}

\usepackage[caption=false,font=footnotesize]{subfig}

\usepackage{algorithm}
\usepackage{algpseudocode}
\usepackage{amssymb}

\algrenewcommand\algorithmicrequire{\textbf{Input:}}
\algrenewcommand\algorithmiccomment[1]{\hfill\(\triangleright\) \textit{#1}}
\DeclareMathOperator{\DP}{DP}
\newcommand{\gA}{\mathsf{A}}
\newcommand{\gB}{\mathsf{B}}

\usepackage{pifont}%

\def\BibTeX{{\rm B\kern-.05em{\sc i\kern-.025em b}\kern-.08em
    T\kern-.1667em\lower.7ex\hbox{E}\kern-.125emX}}
\usepackage{ifthen}
\newboolean{showcomments}
\setboolean{showcomments}{false}
\ifthenelse{\boolean{showcomments}}
{ \newcommand{\mynote}[3]{
		\fbox{\bfseries\sffamily\scriptsize#1}
		{\small$\blacktriangleright$\textsf{\emph{\color{#3}{#2}}}$\blacktriangleleft$}}
	\newcommand{\zzz}[1]{{\setlength{\fboxsep}{2pt}\fcolorbox{black}{yellow}{\textsf{\emph{#1}}}}\xspace}}
{ \newcommand{\mynote}[3]{}
	\newcommand{\zzz}[1]{}}

\newcommand{\sayan}[1]{\mynote{Sayan}{#1}{cyan}}

\usepackage{acronym}
\acrodef{DL}{decentralized learning}
\acrodef{ML}{machine learning}
\acrodef{D-PSGD}{decentralized parallel stochastic gradient descent}
\acrodef{FL}{federated learning}
\acrodef{SGD}{stochastic gradient descent}
\acrodef{IID}{independent and identically distributed}
\acrodef{non-IID}{non independent and identically distributed}
\acrodef{RMSE}{root mean square error}
\acrodef{RMW}{random model walk}
\acrodef{GL}{gossip learning}
\acrodef{EL}{epidemic learning}
\acrodef{DWT}{discrete wavelet transform}
\acrodef{FFT}{fast Fourier transform}
\acrodef{MI}{mutual information}
\acrodef{VN}{virtual node}
\acrodef{RN}{real node}
\acrodef{LDP}{local differential privacy}
\acrodef{PNDP}{pairwise network differential privacy}
\acrodef{PNLDP}{pairwise network local differential privacy}
\acrodef{GI}{gradient inversion}
\acrodef{CML}{collaborative machine learning}
\acrodef{TPR}{true positive rate}
\acrodef{FPR}{false positive rate}
\acrodef{LLM}{large language model}
\acrodef{DP}{demographic parity}

\acrodef{LA}{linkability attack}
\acrodef{GIA}{gradient inversion attack}
\acrodef{MIA}{membership inference attack}
\acrodef{AIA}{attribute inference attack}
\acrodef{ROC}{receiver operating characteristic}
\acrodef{AUC}{area under the ROC curve}
\acrodef{LAS vector}{locally aggregated statistics vector}
\acrodef{GAS vector}{globally aggregated statistics vector}
\acrodef{ASR}{attack success rate} %
\newcommand{\sys}{\textsc{Auditopus}\xspace}

\newcommand{\civil}{Civil Comments\xspace}
\newcommand{\biasbios}{Bias in Bios\xspace}
\newcommand{\qwen}{\textsc{Qwen2.5-7B-Instruct}\xspace}

\newcommand{\llama}{\textsc{Llama-3.1-8B-Instruct}\xspace}

\newtcolorbox{promptbox}[1][]{%
  enhanced, breakable,
  colback=gray!5, colframe=gray!60!black,
  boxrule=0.5pt, arc=2pt,
  left=4pt, right=4pt, top=3pt, bottom=3pt,
  fonttitle=\bfseries\small, coltitle=white, colbacktitle=gray!60!black,
  fontupper=\ttfamily\footnotesize,
  #1} %
\graphicspath{ {figures/} }
\usepackage{tikz}
\usepackage[eulergreek]{sansmath}
\usepackage{pgfplots}
\usepackage{pgfplotstable}
\usepackage{cleveref}
\usepackage{comment}
\usepackage{siunitx}

\pgfplotsset{compat=newest}
\usepgfplotslibrary{external,units,colorbrewer,groupplots,fillbetween,statistics}
\usetikzlibrary{patterns,shapes.misc}

\pgfplotsset{compat=newest,
    width=\linewidth, height=3.5cm,
    every axis/.append style={
     semithick,
     ymajorgrids=true,
    },
    legend style={
     draw=none,
     fill=white,
     anchor=south,
     at={(0.5, 1)},
     font=\small,
    },
	{cycle list/Set1},
    cycle list name=Set1,
    cycle multiindex* list={
    	Set1\nextlist
    	{dashed,solid,dotted,dashdotted,densely dashdotdotted}\nextlist
    	mark list*\nextlist
    },
    legend cell align=left,
    xtick align=outside,
    xtick pos=bottom,
    major tick length=1mm,
    /pgfplots/short line legend/.style={
     legend image code/.code={
         \draw[mark repeat=2,mark phase=2,##1]
             plot coordinates {
                 (0cm,0cm)
                     (0.2cm,0cm)
                     (0.4cm,0cm)
             };%
     }
    },
    /pgfplots/ybar legend/.style={
    /pgfplots/legend image code/.code={%
         \draw[##1,/tikz/.cd,yshift=-0.25em]
             (0cm,0cm) rectangle (4pt,6pt);},
    },
    /pgfplots/xbar legend/.style={
    /pgfplots/legend image code/.code={%
         \draw[##1,/tikz/.cd,yshift=-0.25em]
             (0cm,0cm) rectangle (4pt,6pt);},
    },
    unit code/.code={\si{#1}},
    tick label style = {font=\sansmath\sffamily},
    every axis label = {font=\sffamily},
    legend style = {font=\sffamily},
    label style = {font=\sffamily},
    select coords between index/.style 2 args={
    	x filter/.code={
    		\ifnum\coordindex<#1\fi
    		\ifnum\coordindex>#2\fi
    }},
	discard if not/.style 2 args={
		x filter/.code={
			\edef\tempa{\thisrow{#1}}
			\edef\tempb{#2}
			\ifx\tempa\tempb
			\else
			
			\fi
	}}
}

\pgfplotsset{table/search path={plots/export/data}}

\crefname{assumption}{assumption}{assumptions}

\tikzsetexternalprefix{figures/}
\tikzset{external/mode=list and make}

\makeatletter
\begingroup\endlinechar=-1\relax
\everyeof{\noexpand}%
\edef\x{\endgroup\def\noexpand\homepath{%
		 }}\x
\makeatother

\def\overleafhome{/tmp}
\newcommand{\inputplot}[2]{%
	\ifx\homepath\overleafhome%
	\IfBeginWith{#1}{plots}{\includegraphics{main-figure#2.pdf}}{#1}%
	\else%
	{\sffamily\scriptsize\input{#1}}
	\fi
}

\newcommand{\newgroupwidth}[2]%
{\expandafter\xdef\csname groupwidth#1\endcsname{#2}}

\newcounter{groupwidth}
\newsavebox{\groupwidthbox}
\makeatletter
{\edef\groupnumber{#1}%
	\stepcounter{groupwidth}%
	\@ifundefined{groupwidth\thegroupwidth}{\pgfmathsetlengthmacro{\mywidth}{\linewidth/\groupnumber}}%
	{\expandafter\let\expandafter\mywidth\csname groupwidth\thegroupwidth\endcsname}%
	\begin{lrbox}{\groupwidthbox}%
		\tikzset{/pgfplots/width={\mywidth}}%
		\ignorespaces}%
	{\end{lrbox}%
	\usebox\groupwidthbox
	\pgfmathsetlengthmacro{\mywidth}{\mywidth + (\linewidth - \wd\groupwidthbox)/\groupnumber}
	\immediate\write\@auxout{\string\newgroupwidth{\thegroupwidth}{\mywidth}}}
\makeatother
\usepackage{amsmath,amssymb,amsfonts,amsthm}
\usepackage{physics}
\usepackage{enumitem}
\usepackage{bbm}
\theoremstyle{definition}

\theoremstyle{remark}

\newtheorem{theorem}{Theorem}
\newtheorem{lemma}[theorem]{Lemma}
\newtheorem{corollary}[theorem]{Corollary}

\allowdisplaybreaks

\newcommand{\maptab}[1]{\mathcal{M}_{#1}}   %
\newcommand{\recvtab}[2]{\mathcal{M}_{#1\to #2}} 
\begin{document}

\title{Robust Decentralized Fairness Auditing}
\makeatletter
\newcommand{\linebreakand}{%
  \end{@IEEEauthorhalign}
  \hfill\mbox{}\par
  \mbox{}\hfill\begin{@IEEEauthorhalign}
}
\makeatother
\author{
  \IEEEauthorblockN{Sayan Biswas}
  \IEEEauthorblockA{\textit{EPFL}, Switzerland\\
    sayan.biswas@epfl.ch}
  \and
  \IEEEauthorblockN{Jade Garcia Bourrée}
  \IEEEauthorblockA{\textit{EPFL}, Switzerland\\
    jade.garcia-bourree@epfl.ch}
  \and
  \IEEEauthorblockN{Anne-Marie Kermarrec}
  \IEEEauthorblockA{\textit{EPFL}, Switzerland\\
    anne-marie.kermarrec@epfl.ch}
  \linebreakand
  \IEEEauthorblockN{Palak}
  \IEEEauthorblockA{\textit{EPFL}, Switzerland\\
    palak.palak@epfl.ch}
  \and
  \IEEEauthorblockN{Martijn de Vos}
  \IEEEauthorblockA{\textit{EPFL}, Switzerland\\
    martijn.devos@epfl.ch}
}

\maketitle

\begin{abstract}
Emerging legislation requires \acp{LLM} to be audited to assess their compliance with regulatory standards, particularly with respect to fairness.
Such black-box audits typically assume a single auditor with access to a large, representative set of queries.
In practice, it can be difficult for an auditor to obtain such a query set, but multiple auditors can together cover the relevant demographic groups by auditing the \ac{LLM} collaboratively with their individual query sets.
However, relying on multiple auditors raises a fundamental trust problem, as they may act on behalf of the \ac{LLM} provider to portray a misleading appearance of fairness, \ie fairwashing.
We propose \sys, a novel approach for robust decentralized fairness auditing. 
In \sys, auditing proceeds in rounds without a central server.
In each round, every auditor issues a fixed number of queries to the \ac{LLM}, and sends only cumulative statistics vectors of its query results to other auditors instead of sensitive queries in clear.
The fairness of the audited \ac{LLM} is then estimated by aggregating all the vectors together.
We show both theoretically and empirically that even a single adversary in the network can steer this estimate by fabricating the vectors it sends, making an unfair \ac{LLM} appear fair.
To address this threat, \sys relies on a defense mechanism where each honest auditor locally down-weights any auditor whose cumulative statistics vectors are statistically inconsistent with previous ones.
We implement \sys and compare it to robust aggregation baselines on two datasets with two pre-trained \acp{LLM}.
Against an attacker that optimizes the vectors it sends to make the \ac{LLM} appear fair, \sys reduces audit error by up to 78\% on average relative to no defense and at least 62\% relative to the robust aggregation baselines. Even when 49\% of auditors are adversarial, \sys never lets a very unfair or moderately unfair LLM pass as fair.
\end{abstract}


\section{Introduction}
\label{sec:intro}

\Acp{LLM} are increasingly used as classifiers in high-stakes applications, including content moderation~\cite{kumar2024watch}, clinical decision support~\cite{thirunavukarasu2023large}, hiring~\cite{armstrong2024silicon}, financial risk assessment~\cite{golec2025interpretable}, and legal document analysis~\cite{guha2023legalbench}.
In these settings, an \ac{LLM}'s decisions directly affect individuals, and systematic disparities in how demographic groups are treated by the \ac{LLM} can cause real harm, ranging from lost economic opportunities (\eg, unjustified loan denials) to personal consequences (\eg, a missed medical diagnosis)~\cite{mcgregor2021preventing}. 
Emerging legislation, such as the EU AI Act~\cite{euaiact2024} and New York City's Local Law~144 on automated employment decision tools~\cite{nycdcwp_aedt}, therefore requires such systems to be audited for compliance, in particular for fairness across demographic groups.
However, these \acp{LLM} are typically proprietary and served through APIs, so auditors have access to neither the \ac{LLM}'s weights nor its training data. They typically only have \emph{black-box access}, \ie they can only submit queries and observe the \ac{LLM}'s decisions~\cite{casper2024black}.
Fairness must therefore be audited in such black-box settings by querying the \ac{LLM} and measuring disparities in its outputs~\cite{armstrong2024silicon}.%

In a black-box fairness audit, a single auditor, such as a regulator, assembles a set of queries, submits them to the \ac{LLM}, and estimates a fairness metric from the \ac{LLM}'s decisions.
A standard choice, and the one we focus on in this work, is \emph{\ac{DP}}~\cite{calders2009building, dwork2012fairness} that measures the difference between the rates at which two demographic groups receive a positive decision from the \ac{LLM}. The reliability of this estimate heavily depends on the query set, which must meet two requirements.
First, it must be \emph{large}: fairness metrics compare decision rates across groups (\eg, male and female), and estimating them precisely requires many queries per group~\cite{ji2020can}.
Second, it must be \emph{representative} of the population the \ac{LLM} serves in deployment; otherwise, the audit measures fairness on the wrong distribution and can over- or underestimate the disparity~\cite{ding2021retiring}.
Meeting both requirements is difficult for a single auditor.
Realistic inputs, such as patient records or job applications, are sensitive and protected by data-protection regulations; public benchmarks are not necessarily representative; and every query to a commercial API has a cost~\cite{andrus2021we}.
In practice, the relevant audit queries are typically spread across many parties, such as hospitals, employers, or civil-society organizations, each holding only a small, possibly unrepresentative query set~\cite{rieke2020future}.

\begin{figure}[t]
    \centering
    \includegraphics[width=.75\linewidth]{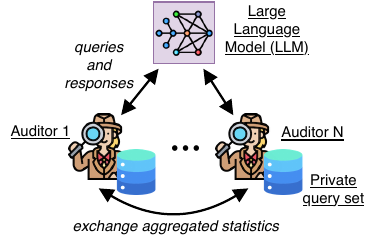}
    \caption{A collaborative fairness audit of a \acf{LLM}.}
    \label{fig:setting}
\end{figure}

Individual auditors, however, can \emph{collaborate} to effectively conduct a fairness audit of an \ac{LLM}.
By auditing the \ac{LLM} collaboratively, auditors can estimate fairness over the union of their query sets, which is far larger and more representative than any single auditor's, while sharing the cost of querying.
This addresses both requirements.
Sharing raw queries, however, is not an option due to their sensitivity.
Fortunately, auditors need not send their queries at all. Group fairness metrics such as \ac{DP} depend only on aggregate statistics, so each auditor can query the \ac{LLM} with its own queries and send only these aggregated statistics to other auditors. %
We show an example of this setting in \Cref{fig:setting} where $N$ auditors audit the fairness of an \ac{LLM} and send aggregated statistics.

Collaborative fairness audits, however, raise two challenges.
First, the auditor query sets are \emph{heterogeneous} in terms of size, groups, and predicted labels. Each query set reflects its own sub-population. Therefore, local estimates may  widely differ from each other~\cite{ding2021retiring} and from the population-level metric, making aggregation of local fairness values inaccurate.
Second, some auditors may be \emph{adversarial}. A participant colluding with the \ac{LLM} provider can send fabricated numbers that steer honest auditors toward a verdict of fairness, a form of fairwashing~\cite{aivodji2019fairwashing}. 
Such adversaries may have a stake in the audit outcome, as platform employees, contractors, or paid affiliates, or because they benefit, economically or ideologically, from the discrimination going undetected.
These two challenges compound each other. Robust aggregation methods typically discard contributions that deviate from the majority, but under heterogeneity, honest auditors with unusual sub-populations also deviate, so discarding outliers also discards the data that makes the audit representative.

We propose \sys, a novel approach for robust decentralized fairness auditing.
Rather than aggregating all statistics in a single exchange, \sys proceeds in multiple rounds.
Each round, every auditor issues a fixed number of queries to the \ac{LLM}, and sends only statistics vectors to a few other auditors, never the sensitive queries themselves.
This round-based structure bounds how much any single auditor can contribute per round, thus limiting the effect an adversary can have in influencing the round-wise estimate. Moreover, it gives each auditor a history of every other auditor's statistics vectors to check new vectors against. This enables \sys to catch subtle fairwashing that a single round would miss (see \Cref{sec:colab_auditing_prelim}).
Each auditor then estimates the \ac{LLM}'s fairness by aggregating all the vectors it has received.
To stress the importance of robustness, we introduce a strong adaptive adversary that attempts fairwashing by fabricating, every round, the statistics vectors that bring the estimated \ac{DP} as close to $0$ as possible. We design a new defense against this threat that is central to \sys.
Rather than trusting all auditors equally, each honest auditor locally down-weights any peer whose per-round reports are statistically inconsistent with its own past reports. An honest auditor, by contrast, samples uniformly from a fixed query set, so its increments remain consistent over time.

We implement and evaluate \sys on two datasets against robust aggregation baselines.
\sys prevents fairwashing of  very unfair and moderate \ac{LLM} on both datasets, even when 49\% of the auditors are adversarial, and reduces the error of the estimated \ac{DP} by up to 78\% on average relative to no defense and by at least 62\% relative to the robust aggregation baselines.
Near-fair \acp{LLM} remain the hardest to protect, since a small shift of the estimated \ac{DP} suffices to flip their verdict, yet \sys still at least halves this error.\sayan{same marketing ++ here as in the abstract}

\textbf{Contributions.} We make the following four contributions:
\begin{enumerate}
  \item We introduce a novel round-based decentralized fairness auditing mechanism, named \sys, in which auditors query the \ac{LLM} with their own private queries and only exchange cumulative statistics vectors with their neighbors (\Cref{sec:design}).
  \item We introduce a strong adaptive adversary that fairwashes the \ac{LLM} by optimizing its fabricated statistics vectors every round, and we theoretically characterize the conditions under which a single adversary can steer the estimate so that an unfair \ac{LLM} appears fair (\Cref{sec:adversary}). %
  \item We design a defense that is central to \sys: each honest auditor locally down-weights any auditor whose statistics vectors are statistically self-inconsistent (\Cref{sec:defense}).
  It requires no trusted party, and unlike robust aggregation techniques, it does not discard honest auditors whose query sets are atypical.
  \item We evaluate \sys on two real-world datasets against robust aggregation baselines (\Cref{sec:experiments}).
    \sys effectively neutralizes the threat from adaptive adversaries present in the network.
    With 23\% adversaries in the network, the error of the estimated \ac{DP} of \sys stays between $0.02$ and $0.05$ as the number of honest auditors and the heterogeneity of their query sets vary, while it reaches up to $0.26$ for robust aggregation baselines, at negligible communication and compute cost. Even with 49\% of the auditors being adversarial, it never lets a very unfair or moderately unfair LLM pass as fair. %
\end{enumerate} %
\section{Background and Problem Description}
\label{sec:background}

We first introduce black-box and collaborative fairness auditing, and then formulate our research problem.

\subsection{Black-box Fairness Auditing}
\label{sec:bb_prelim}
The goal of a fairness audit is to verify whether a deployed \ac{LLM} treats different demographic groups equitably, for example, to assess its compliance with regulation.
Such an audit involves three entities:
\begin{enumerate*}[label=(\textit{\roman*})]
  \item the \emph{platform}, which hosts the \ac{LLM} and uses it as a classifier;
  \item the \emph{users}, who interact with the service offered by the platform; and
  \item the \emph{auditor}, which conducts the audit to verify whether the \ac{LLM} is fair across all users.
\end{enumerate*}
The auditor can be, for example, a state regulator, a consulting firm, or a group of users.
In a black-box audit, the auditor has no access to the \ac{LLM}'s internals, such as its architecture, weights, or training data. Regulations such as the EU AI Act ~\cite{euaiact2024} do require platforms to disclose information about their systems, \eg technical documentation or transparency reports.
However, this information is self-reported and cannot be independently verified by the auditor, and it may be incomplete or inconsistent with the platform's actual behavior~\cite{trujillo2025dsa, garcia2025robust}.
Following~\cite{garcia2025robust}, we therefore treat the query-answer pairs collected by the auditor as the only trusted source of information, and disregard any public or platform-provided information about the system.

\textbf{Fairness metric.} In this work, we consider black-box audits targeting the \textit{fairness} of an \ac{LLM} $f$.
Among all fairness metrics, we study \textit{Demographic Parity (DP)} \cite{calders2009building, dwork2012fairness} which is commonly used in the fairness evaluation literature.
We consider an \ac{LLM} used as a binary classifier.
For an input $X$, it outputs a decision $f(X) \in \{0, 1\}$, where $f(X) = 1$ denotes a positive decision (\eg, a comment flagged as toxic).
Each input belongs to one of two demographic groups $\gA$ and $\gB$, defined by a sensitive attribute such as religion or gender. A multi-class or multi-group classifier reduces to this setting by auditing one class against all others, as we do in \Cref{sec:experiments}.
For an input $X$ drawn from the population served by $f$, the DP of $f$ is defined as follows:
\begin{equation}\label{eq:estimand}
  \DP(f) = \mathbb{P}(f(X) = 1 \mid X \in \gA)
        - \mathbb{P}(f(X) = 1 \mid X \in \gB) .
\end{equation} 
We focus on DP as it only depends on the \ac{LLM}'s outputs, and can thus be estimated from query-answer pairs alone.
Nevertheless, our approach only relies on the fact that the metric is a difference between conditional probabilities of the \ac{LLM}'s output, and therefore extends to other group fairness metrics with the same structure.
For instance, Equality of Opportunity~\cite{hardt2016equality} is obtained by additionally conditioning on $Y = 1$ in Eq.~\eqref{eq:estimand}, and Equalized Odds~\cite{hardt2016equality} by considering both $Y = 1$ and $Y = 0$, provided that the auditor knows the ground-truth labels $Y$ of its queries.

\textbf{Fairness verdict.}
The auditor estimates $\DP(f)$ from its query-answer pairs and outputs a \emph{verdict}.
The \ac{LLM} is deemed \emph{fair} if the estimate $\widehat{\DP}$ lies in the \emph{fair band} $[-\tau, \tau]$, \ie, $|\widehat{\DP}| \le \tau$ for a tolerance $\tau$ set by the entity commissioning the audit, and \emph{unfair} otherwise.

\begin{figure}[t]
    \centering
    \includegraphics[width=.65\linewidth]{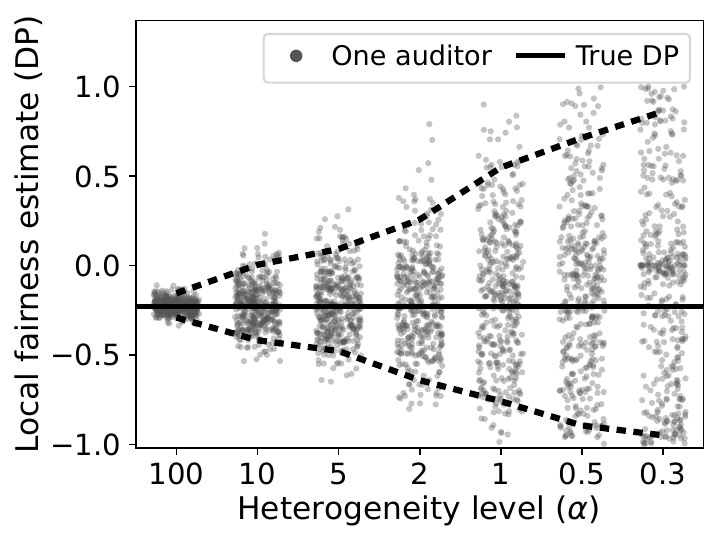}
    \caption{Local \ac{DP} estimations of $N = 20$ honest auditors on the \civil dataset as data heterogeneity grows (lower $\alpha$ = more heterogeneous).
    The solid horizontal line indicates the true \ac{DP} and the dotted lines indicate the 5th and 95th percentiles. Heterogeneity of private query sets makes local fairness estimates unreliable.}
    \label{fig:collaborative-auditing}
\end{figure}

\subsection{Collaborative Fairness Auditing}
\label{sec:colab_auditing_prelim}
As discussed in \Cref{sec:intro}, a single auditor rarely holds a query set that is large and representative of the population.
We therefore consider $N$ auditors, all honest for now, where each auditor $i$ holds a private query set $\mathcal{D}_i$ drawn from its own sub-population, such as the patients of a hospital or the applicants of an employer.
In this collaborative setting, each auditor $i$ queries the \ac{LLM} with its own queries in $\mathcal{D}_i$ and summarizes the resulting decisions into a \emph{\ac{LAS vector}} $c_i = (a_i, A_i, b_i, B_i)$, where $A_i$ is the number of queries from group~$\gA$, $a_i$ the number of those receiving a positive decision, and $B_i$ and $b_i$ their group-$\gB$ counterparts.
For any such vector $c = (a, A, b, B)$, the \ac{DP} of the queries is:
\begin{equation}
\label{eq:local_dp}
\DP(c) = \frac{a}{A} - \frac{b}{B}.
\end{equation}
Applied to its own \ac{LAS vector}, this gives auditor~$i$ its \emph{local} estimate $\DP(c_i)$.
We also apply it below to sums of \acp{LAS vector}.
Crucially, \ac{DP} depends only on these \acp{LAS vector}, so auditors can collaborate by exchanging them rather than their queries.

\textbf{Summing \acp{LAS vector} under heterogeneity.}
Heterogeneity makes local \ac{DP} estimates unreliable.
Figure~\ref{fig:collaborative-auditing} shows the local \ac{DP} estimates of 20 honest auditors on the \civil dataset, with \qwen as the \ac{LLM} classifier, when the data is split across auditors with a Dirichlet distribution of concentration~$\alpha$ (the smaller~$\alpha$, the more auditors differ in group composition and positive rate).
With near-IID data ($\alpha = 100$), all local estimates lie close to the true DP.
As $\alpha$ decreases, they spread out. The gap between the 5th and 95th percentiles widens from $0.14$ at $\alpha = 100$ to $1.81$ at $\alpha = 0.3$, and many auditors even observe a disparity with the opposite sign.
Some auditors may still land on the correct verdict, but only by chance. From its own query set alone, an auditor cannot tell whether its local estimate reflects the \ac{LLM}'s disparity or merely the composition of its query set. Combining the local estimates, \eg by their mean or median, does not solve this.
It weights all auditors equally in both groups, although their proportions of each group differ.
When group proportions differ across sub-populations, rates computed per sub-population and on the aggregated data can disagree, and even point in opposite directions (Simpson's paradox)~\cite{bickel1975sex}.
Instead, auditors can sum their \acp{LAS vector} into a single \emph{\ac{GAS vector}} $C = \sum_i c_i$ and compute a single estimate from it:
\begin{equation}
\begin{aligned}
  \DP(C) &= \frac{\sum_i a_i}{\sum_i A_i} - \frac{\sum_i b_i}{\sum_i B_i}\\
         &= \sum_i \frac{A_i}{\sum_j A_j}\,\frac{a_i}{A_i}
          - \sum_i \frac{B_i}{\sum_j B_j}\,\frac{b_i}{B_i}.
\end{aligned}
\label{eq:pooled_dp}
\end{equation}
This weights each auditor's positive rate in a group by that auditor's proportion of the group.
It is equivalent to a single auditor holding the union of all query sets, and thus recovers the DP of this union exactly, whatever the data heterogeneity.
Collaborative audits should therefore estimate DP from the \ac{GAS vector}, and \sys does so.

\begin{figure}[t]
    \centering
    \includegraphics[width=0.65\linewidth]{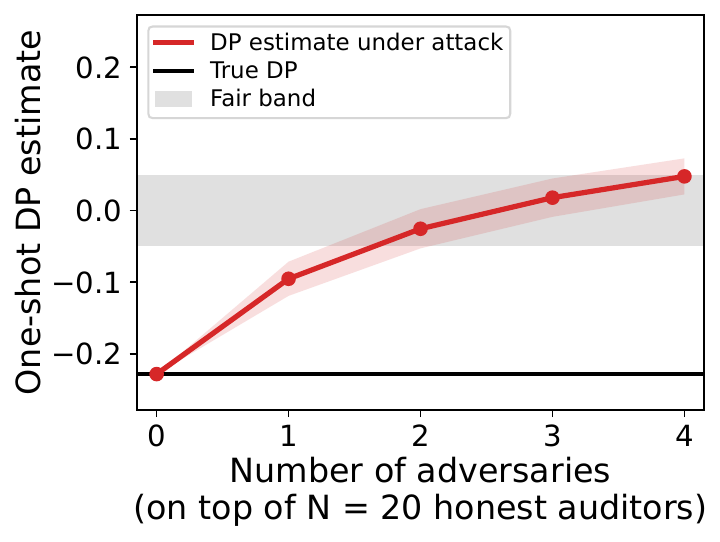}
    \caption{One-shot auditing on the \civil dataset (Christian vs.\ Muslim, true DP $-0.228$), with $N = 20$ honest auditors holding heterogeneous data ($\alpha = 1$). Each adversary submits fabricated counts as large as the largest honest auditor's, all non-toxic comments about group $\mathsf{B}$. Mean $\pm$ SD over 20 seeds.}
    
    \label{fig:one_shot_audit}
\end{figure}

\textbf{Naive collaborative auditing (one-shot).}
In the simplest collaborative protocol, each auditor queries the \ac{LLM} with its full query set once, sends its \ac{LAS vector} to others, sums all \acp{LAS vector} it receives into a \ac{GAS vector}, and computes \Cref{eq:pooled_dp} on it.
This \emph{one-shot} protocol, as used in prior work~\cite{de2024fairness}, is accurate when all auditors are honest, but extremely fragile otherwise.
An adversary colluding with the platform can send a fabricated \ac{LAS vector} with arbitrarily large $A_i$ and $B_i$, and thereby steer the estimate to any value, including one such that the \ac{LLM} appears fair.
This vulnerability comes from two facts:
\begin{enumerate*}[label=(\textit{\roman*})]
\item honest auditors cannot verify another auditor's \ac{LAS vector}, as they never see its queries; and 
\item nothing bounds the size of a single contribution.
\end{enumerate*}
\Cref{fig:one_shot_audit} illustrates this on the \civil dataset for an audit with $N=20$ honest auditors.
Even two adversaries in a setting with 22 total auditors can move the estimate from $-0.23$ (the true \ac{DP}) into the fair band, and a third pushes it past 0.
They thus control the estimate rather than merely bias it.
Their \acp{LAS vector} are also plausible, since heterogeneous data can produce a group-B-only query set (\Cref{fig:collaborative-auditing}), so nothing distinguishes them from honest ones in a single exchange.
Standard robust aggregation techniques do not solve this.
Methods such as the median or trimmed mean~\cite{yin2018byzantine} discard contributions that deviate from the majority, but under heterogeneity honest auditors with unusual sub-populations deviate too, so they discard the very data that makes the audit representative.
This is why \sys proceeds in rounds with a fixed per-round budget $n$ (\Cref{sec:design}).

\subsection{Problem Description}
\label{sec:problem_description}

In summary, we seek a decentralized auditing protocol in which every honest auditor
\begin{enumerate*}[label=(\roman*)]
  \item estimates the \ac{DP} of the honest auditors' \ac{GAS vector} accurately, even under heterogeneity;
  \item reaches the correct fair/unfair verdict despite adversaries that fabricate their \acp{LAS vector} to fairwash the \ac{LLM};
  \item does not rely on a trusted party; and
  \item sends only \acp{LAS vector}, never its (possibly sensitive) queries. %
\end{enumerate*}

\section{Design of \sys}
\label{sec:design}

In this section, we give a detailed overview of \sys, our novel approach for robust decentralized fairness auditing.

\begin{figure*}[t]
    \centering
    \includegraphics[width=.92\linewidth]{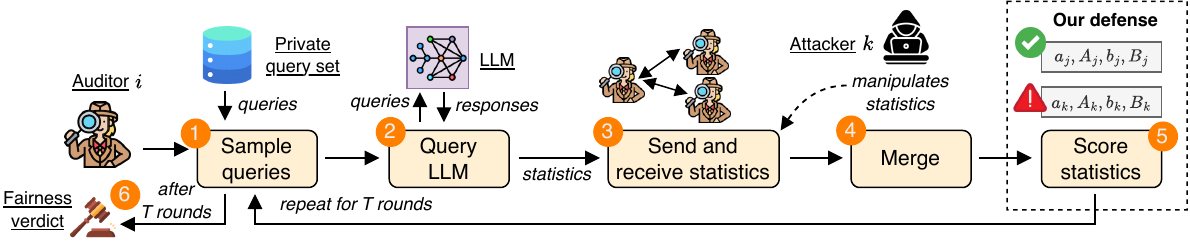}
    \caption{The workflow of \sys, from the perspective of an honest auditor $i$.}
    \label{fig:workflow}
\end{figure*}

\subsection{\textsc{Auditopus} in a Nutshell}
\label{sec:nutshell}

In \sys, a set of auditors $\mathcal{V}$ collaboratively audit a black-box \ac{LLM} over several rounds.
In each round, every auditor queries the \ac{LLM} with a few samples from its private query set.
It then sends only its \ac{LAS vector}, \ie the number of queries and positive predictions per group, never the queries themselves.
Each auditor sends its \ac{LAS vector} to a select few other auditors, and every auditor sums all the \acp{LAS vector} it has received into a \ac{GAS vector} to estimate the \ac{LLM}'s \ac{DP}. The auditing algorithm is described in \Cref{sec:protocol}.

Using the \ac{GAS vector} makes the estimate accurate even when auditors hold very different query sets.
However, it also lets an attacker fabricate \acp{LAS vector} to make an unfair \ac{LLM} look fair.
We elaborate on this threat in \Cref{sec:adversary}.
As a defense, \sys therefore checks each auditor against its own past \acp{LAS vector} rather than against the others.
An honest auditor samples from fixed query sets, so its \ac{LAS vector} stays consistent over rounds.
An attacker must keep adjusting its \acp{LAS vector} to cancel the honest ones.
Each honest auditor then gives less weight to increments that are inconsistent with the source's past \ac{LAS vector}.
We describe this defense in detail in \Cref{sec:defense}.

\subsection{System and Threat Model}
\label{sec:system_model}

\textbf{Auditors and rounds.}
A set of auditors $\mathcal{V}$, of which $N$ are honest and $K$ are adversarial, audits an \ac{LLM} $f$ for \ac{DP} between two groups $\gA$ and $\gB$.
Each honest auditor $i$ holds a private query set $\mathcal{D}_i$, and these query sets are heterogeneous in size and group composition (\Cref{sec:colab_auditing_prelim}).
The audit proceeds in synchronous rounds $t = 1, \dots, T$.

\textbf{Audit goal.}
We assume that the honest auditors' query sets jointly form the population served by $f$, so that the DP of their GAS vector, \ie \eqref{eq:pooled_dp} computed over honest auditors only, equals $\mathrm{DP}(f)$. At the end of the audit, each honest auditor outputs an estimate of $\mathrm{DP}(f)$ and the corresponding verdict. The goal is for every honest auditor to reach the same verdict as $\mathrm{DP}(f)$.
How the individual outcomes are combined into a single audit outcome is left to the entity commissioning the audit, \eg, a regulator.
In \Cref{sec:experiments} we consider the median among all estimates.

\textbf{Network model.}
A central aggregator could collect all \acp{LAS vector} and compute \Cref{eq:pooled_dp}, but it would become a single point of trust and failure. Auditors from different organizations or jurisdictions may not agree on a trusted party, and a compromised aggregator could dictate the audit's outcome.
We therefore consider a decentralized setting, in which auditors communicate over a static, connected graph $\mathcal{G}$ and exchange \acp{LAS vector} only with their neighbors $\mathcal{N}(i)$.
The network is permissioned. Auditors are approved before joining and know each other's identities, so an adversary cannot join under many identities (\ie a Sybil attack~\cite{douceur2002sybil}).
Auditors stay online during the audit, and we assume reliable communication channels.
Each auditor digitally signs its own \ac{LAS vector} together with the round in which it produced it, and auditors relay these signed entries unchanged, so vectors relayed by other auditors cannot be altered or forged.

\textbf{Threat model.}
Adversaries aim at \emph{fairwashing}~\cite{aivodji2019fairwashing}, trying to make an unfair \ac{LLM} appear fair, \eg, on behalf of the platform (\Cref{sec:intro}). An adversary follows the protocol and its message format, but sends fabricated \acp{LAS vector} instead of those of its own queries, and relays the \acp{LAS vector} of other auditors unchanged. Each adversary colludes with the platform, but adversaries do not collude with one another: they do not share their query sets, their \acp{LAS vector}, or their strategies, and each one crafts its reports on its own. Finally, we assume that the platform answers audit queries like regular traffic, since they come from the auditors' private query sets and cannot be distinguished from those of regular users~\cite{godinot2024under}. In particular, the platform cannot single out the honest auditors' queries to infer their statistics. We formalize the adversary's capabilities and strategy in \Cref{sec:adversary}.

In \sys, auditors exchange only \acp{LAS vector}, their own and those they relay, but never their queries. This is intuitively more private than sharing queries in clear. Deployments that need stronger, formal privacy guarantees can perturb each \ac{LAS vector} with a differentially private mechanism~\cite{dwork2006calibrating} before it is sent, or run \sys' checks and aggregations under secure multi-party computation~\cite{bonawitz2017practical}. These remain orthogonal to how \sys fundamentally works.

\begin{algorithm}[t]
\caption{The \sys workflow, as executed by honest auditor $i$.}
\small
\label{alg:sys}
\begin{algorithmic}[1]
\Require query set $\mathcal{D}_i$; auditor set $\mathcal{V}$; neighbors $\mathcal{N}(i)$; per-round query budget $n$; rounds $T$; score cap $z_{\max}$; fair band $\tau$; black-box \ac{LLM} $f$
\State $c_i \gets (0,0,0,0)$ \Comment{Initialize \ac{LAS vector} of $i$}
\State $\maptab{i}[j] \gets \big((0,0,0,0), 0\big)$;\ \ $\mathit{prev}_j, W_j \gets (0,0,0,0)$ for every $j \in \mathcal{V}$
    \Comment{table and \acp{LAS vector} of all auditors}
\For{$t = 1, \dots, T$}
    \State $Q \sim \mathcal{D}_i$, $|Q| = n$;\ \ $\mathcal{D}_i \gets \mathcal{D}_i \setminus Q$ \Comment{Step 1: sample}
    \State $c_i \gets c_i + \Call{QueryLLM}{f, Q}$ \Comment{Step 2: query \ac{LLM}}
    \State $\maptab{i}[i] \gets (c_i, t)$ \Comment{\ac{LAS vector} and its round}
    \State \textbf{Send} $\maptab{i}$ to every $u \in \mathcal{N}(i)$ \Comment{Step 3: send and recv.}
    \State \textbf{Receive} table $\recvtab{u}{i}$ from each $u \in \mathcal{N}(i)$
    \For{each $u \in \mathcal{N}(i)$}
        \Comment{Step 4: merge}
        \For{each $j \in \mathcal{V}$}
            \If{$\recvtab{u}{i}[j].r > \maptab{i}[j].r$}
                \State $\maptab{i}[j] \gets \recvtab{u}{i}[j]$ \Comment{keep the newest vector}
            \EndIf
        \EndFor
    \EndFor
    \For{each $j \in \mathcal{V}$}
        \Comment{Step 5: score}
        \State $\delta \gets \maptab{i}[j].c - \mathit{prev}_j$
            \Comment{$\delta = (\Delta a, \Delta A, \Delta b, \Delta B)$}

            \If{$\Delta A + \Delta B > 0$}
        \If{$\mathit{prev}_j = (0,0,0,0)$}
            \State $w \gets 1$
                \Comment{first increment: no history}
        \Else
            \State $S \gets \Call{Suspicion}{\delta, \mathit{prev}_j}$
            \State $w \gets (1+S)^{-2}$
        \EndIf
        \State $W_j \gets W_j + w\,\delta$
    \EndIf
        \State $\mathit{prev}_j \gets \maptab{i}[j].c$
    \EndFor
    \State $\widehat{\DP}_i \gets \DP\big(\sum_{j \in \mathcal{V}} W_j\big)$ \Comment{Step 6: estimate}
\EndFor
\State \Return \emph{fair} if $|\widehat{\DP}_i| \le \tau$, else \emph{unfair}
\end{algorithmic}
\end{algorithm} 
\subsection{\sys Auditing Workflow}
\label{sec:protocol}
We now describe the workflow of \sys from the perspective of an honest auditor $i$, as illustrated in \Cref{fig:workflow} and detailed in \Cref{alg:sys}.

\textbf{State.}
Each auditor $i$ maintains a \ac{LAS vector} $c_i$, which sums the statistics of all queries it has issued so far.
In addition, it maintains a table $\maptab{i}$ that maps every auditor $j \in \mathcal{V}$ to the latest \ac{LAS vector} that $i$ has received from $j$, stored as an entry $\maptab{i}[j] = (c, r)$ signed by $j$, where $\maptab{i}[j].c$ is the \ac{LAS vector} and $\maptab{i}[j].r$ the round in which $j$ produced it.
To score incoming statistics, $i$ also keeps, for every auditor $j$, the \ac{LAS vector} $\mathit{prev}_j$ that it held for $j$ at the end of the previous round, and a weighted \ac{LAS vector} $W_j$ that accumulates $j$'s contributions to $i$'s estimate.
All vectors are initialized to zero, and all entries to round $0$ (Lines~1-2).
Each round then consists of the following steps.

\textbf{Step 1: Sample queries.}
At the start of round $t$, auditor $i$ draws a set $Q$ of $n$ queries uniformly at random from $\mathcal{D}_i$ and removes them from its query set, so that no query is issued twice (Line~4).

\textbf{Step 2: Query the \ac{LLM}.}
Auditor $i$ submits the queries in $Q$ to the black-box \ac{LLM} $f$, records the decisions, and adds the resulting statistics to its \ac{LAS vector} $c_i$. It then stores $c_i$ in its own entry $\maptab{i}[i]$, signed and stamped with the current round $t$ (Lines~5-6).

\textbf{Step 3: Send and receive statistics.}
Auditor $i$ sends its entire table $\maptab{i}$ to each of its neighbors in $\mathcal{G}$, and receives the table $\recvtab{u}{i}$ of each neighbor $u \in \mathcal{N}(i)$ in return (Lines~7-8).
Auditors thus relay the (signed) entries of other auditors in addition to their own, so the statistics of each auditor reach every other auditor, advancing one hop per round.
Only \acp{LAS vector} are exchanged, never queries.

\textbf{Step 4: Merge.}
Auditor $i$ merges each received table into its own, entry by entry. It replaces $\maptab{i}[j]$ with the received entry $\recvtab{u}{i}[j]$ if the latter was produced in a later round and carries a valid signature of $j$ (Lines~9-15).
Relaying whole signed entries rather than per-round statistics makes the merge robust to gossip.
An auditor may receive the same entry along several paths or with a delay of several rounds, and comparing rounds neither aggregates it twice nor lets an older vector overwrite a newer one.

\textbf{Step 5: Score statistics.}
Without a defense, $i$ would aggregate the latest vectors of all auditors into $\sum_{j} \maptab{i}[j].c$, whose \ac{DP} equals $\mathrm{DP}(f)$ once all statistics have arrived, provided every auditor is honest (\Cref{sec:colab_auditing_prelim}).
As we show in \Cref{sec:adversary}, however, a malicious auditor can fabricate its statistics to steer this estimate into the fair band.
\sys therefore weighs the statistics of each auditor before summing them.
For each auditor $j$, including $i$ itself, $i$ computes the \emph{increment} $\delta = \maptab{i}[j].c - \mathit{prev}_j$, \ie, the statistics that $j$ added since $i$ last updated its entry (Line~17).
When $j$ sends new \acp{LAS vector}, $i$ assigns $\delta$ a suspicion score $S$ that measures how statistically inconsistent $\delta$ is with $j$'s \ac{LAS vector} $\mathit{prev}_j$, and adds $\delta$ to $W_j$ with weight $w = (1+S)^{-2}$ (Lines~18-25). The first increment of $j$ has no history to compare with, so $i$ gives it weight $1$ without scoring it (Lines~19-20).
The intuition is that an honest auditor samples uniformly from a fixed query set, so its increments remain consistent with its own \ac{LAS vector}, whereas an attacker must keep adjusting its fabricated statistics to counter the honest ones. %
Since each auditor is compared to its own \ac{LAS vector} rather than to the other auditors, honest auditors whose query sets differ from the majority are not penalized.
The weight of an increment is fixed once assigned and never revised.
We detail the suspicion score and the weighting in \Cref{sec:defense}.

\textbf{Step 6: Estimate and verdict.}
Auditor $i$ computes its estimate $\widehat{\DP}_i = \DP\bigl(\sum_{j \in \mathcal{V}} W_j\bigr)$ from the weighted \acp{LAS vector} of all auditors (Line~29).
After $T$ rounds, $i$ outputs the verdict \emph{fair} if $|\widehat{\DP}_i| \le \tau$, and \emph{unfair} otherwise (Line~31).

\begin{figure}[t]
    \centering
    \includegraphics[width=0.65\linewidth]{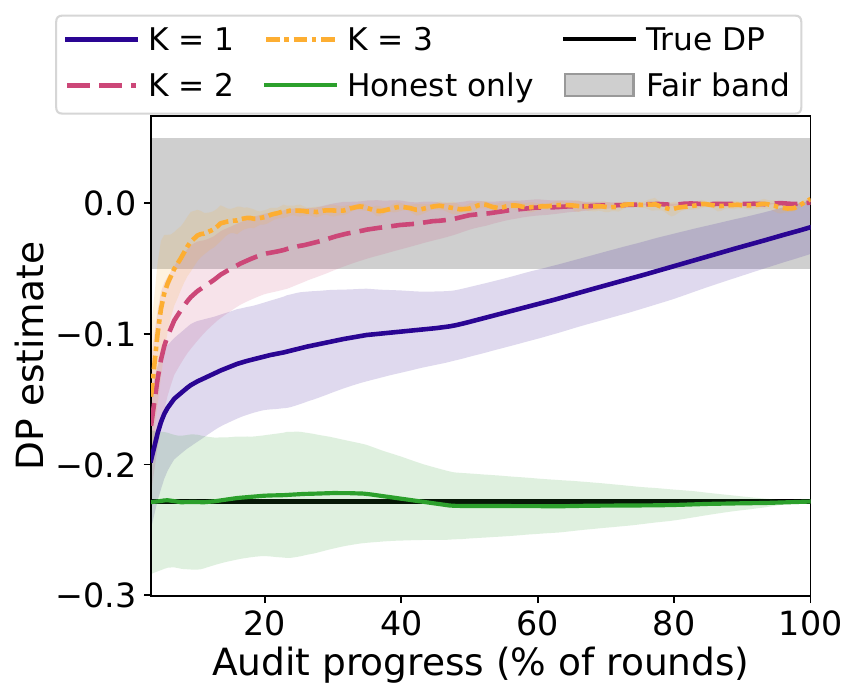}
    \caption{Estimate of $\DP$ with $K = 1, 2, 3$ adversaries and no defense, on \civil (Christian vs.\ Muslim), with $N = 20$ honest auditors, $\alpha = 1$ and $n = 300$ (20 seeds).}
    \label{fig:multiple_adversary_no_defense}
\end{figure}

\begin{algorithm}[t]
\caption{Fairwashing attacker $k$: \Cref{alg:sys} with Steps~1--2 replaced by \textsc{FabricateStats}, and Steps~5--6 omitted.}
\small
\label{alg:attacker}
\begin{algorithmic}[1]
\Require auditor set $\mathcal{V}$; neighbors $\mathcal{N}(k)$; per-round query budget $n$; rounds $T$
\State $c_k \gets (0,0,0,0)$;\ \ $\maptab{k}[j] \gets \big((0,0,0,0), 0\big)$ for every $j \in \mathcal{V}$
\For{$t = 1, \dots, T$}
    \State $\delta \gets \Call{FabricateStats}{\maptab{k}, t}$
        \Comment{replaces Steps~1--2}
    \State $c_k \gets c_k + \delta$;\ \ $\maptab{k}[k] \gets (c_k, t)$
    \State \textbf{Send} $\maptab{k}$ to every $u \in \mathcal{N}(k)$
    \State \textbf{Receive and merge} as in Algorithm~\ref{alg:sys}, Steps~3--4
        \Comment{relays entries unchanged}
\EndFor
\Statex
\Function{FabricateStats}{$\maptab{k}, t$}
    \Comment{returns increment $\delta = (\Delta a, \Delta A, \Delta b, \Delta B)$}
    \If{$t = 1$}
        \State $\Delta A \gets \lfloor n/2 \rfloor$;\ \ $\Delta B \gets n - \Delta A$
            \Comment{nothing heard yet}
        \State \Return $\big(\lfloor \Delta A/2 \rfloor,\ \Delta A,\ \lfloor \Delta B/2 \rfloor,\ \Delta B\big)$
            \Comment{neutral vector}
    \Else
        \State $C \gets \sum_{j \in \mathcal{V}} \maptab{k}[j].c$
            \Comment{attacker's view of the \ac{GAS vector}}
        \State \Return $\delta^\star \in \operatorname*{arg\,min}_{\delta \in \mathcal{X}_n} |\DP(C + \delta)|$
            \Comment{\eqref{eq:attack}; $\mathcal{X}_n$: feasible moves, \eqref{eq:feasible}}
    \EndIf
\EndFunction
\end{algorithmic}
\end{algorithm}  %
\section{Adversary}
\label{sec:adversary}

After defining the adaptive attacker, we reduce its problem to a single variable, its budget split between the two groups, characterize when one attacker can fairwash the audit, and extend the analysis to several non-colluding adversaries.

\subsection{Strategy of an Adversary}
\label{sec:adversary:threat}

An adversary $k$ joins the graph like any other auditor and follows the message format. Each round it adds an increment $\delta$ to its \ac{LAS vector}, $c_k \gets c_k + \delta$, but $\delta$ is fabricated. Honest auditors cannot verify it, because they never see another auditor's raw queries. They can only check that the \ac{LAS vector} is syntactically valid.

\textbf{Admissible increments.}\label{as:valid}
At every round, the increment $\delta = (\Delta a, \Delta A, \Delta b, \Delta B)$ sent by an adversary satisfies $\delta \in \mathcal{X}_{n}$, with
\begin{equation}\label{eq:feasible}
\begin{aligned}
  \mathcal{X}_{n} = \big\{ \delta \in \mathbb{Z}_{\geq 0}^{4} :\ 
    & \Delta A + \Delta B = n,\\
    & \Delta a \leq \Delta A,\ \Delta b \leq \Delta B \big\} .
\end{aligned}
\end{equation}
Within these bounds the entries are arbitrary, the adversary need not query $f$ at all, and may fabricate \acp{LAS vector}.

\textbf{Per-round objective under past information.}\label{as:info}
An adversary sends its \ac{LAS vector} at every round from $t = 1$ onward. When acting at round $t$, it knows its table $\maptab{k}$, whose entries date from rounds strictly before $t$. We collect them in the \ac{GAS vector}:
\begin{equation}\label{eq:history}
        C = \sum_{j \in \mathcal{V}} \maptab{k}[j].c = (a, A, b, B),
\end{equation}
which includes its own past fabricated \ac{LAS vector}. Its objective is the estimate it can compute at that moment, namely the \ac{GAS vector} corrected by its own current \ac{LAS vector}:
\begin{equation}\label{eq:objective}
\DP(C + \delta) =
    \frac{a + \Delta a}{A + \Delta A}
    -
    \frac{b + \Delta b}{B + \Delta B} .
\end{equation}

We instantiate a strong, adaptive attacker. In every round $t \geq 2$, it picks the increment that brings its objective closest to $0$
\begin{equation}\label{eq:attack}
    \delta^\star \in \operatorname*{arg\,min}_{\delta \in \mathcal{X}_{n}} \big|\DP(C + \delta)\big| .
\end{equation}

\textbf{Attacker workflow.}
Algorithm~\ref{alg:attacker} details how adversary $k$ runs this strategy. It follows the round structure of Algorithm~\ref{alg:sys}, except that it never queries the LLM. Instead of sampling and querying (Steps~1--2), it calls \textsc{FabricateStats} (Line~3). It then adds the returned increment $\delta$ to its own \ac{LAS vector} and stores the result in its entry $\maptab{k}[k]$ (Line~4). Finally, it sends and merges tables exactly like an honest auditor (Lines~5-6), relaying the other auditors' entries unchanged, so that its messages have the same format as honest ones. It skips the scoring and estimation steps (Steps~5--6), which play no role in its attack. \textsc{FabricateStats} distinguishes two cases. In the first round, the adversary has not received any table yet and has nothing to steer, so it sends a neutral \ac{LAS vector}.
An even split of its budget between the two groups, with positive rate $1/2$ in each group (Lines~9-11). %
From round~$2$ onwards, it sums its table into its view $C$ of the \ac{GAS vector}~\eqref{eq:history} (Line~13) and returns the move $\delta^\star$ of~\eqref{eq:attack} (Line~14). Computing $\delta^\star$ requires no query to the LLM, only a search over the finite set $\mathcal{X}_n$.

\textbf{Remarks.} First, the objective is evaluated \emph{at each round separately}, so the analysis is independent of the stopping rule, be it a fixed horizon or a convergence criterion. We therefore drop the round index from $C$ and $\delta$. Second, the \ac{GAS vector} \eqref{eq:history} accumulates the adversary's own past fabrications, so its influence does not dilute as rounds pass. Acting from round $1$, it owns a proportion of at least $1/(N+K)$ of the total mass at every round. Third, \eqref{eq:objective} is only a proxy for the estimate that honest auditors actually compute, since the attacker does not yet know the most recent honest \acp{LAS vector}. These recent \acp{LAS vector} are however a fraction of all \acp{LAS vector} accumulated so far, so the gap shrinks as the audit proceeds and the attacker's objective becomes an accurate proxy of the honest estimate.

\subsection{Reduction to a Budget Split}
\label{sec:adversary:move}

We assume $A > 0$ and $B > 0$ so that the ratios in \eqref{eq:objective} are well defined, which holds from round $2$ on as soon as the honest auditors cover both groups. We assume, without loss of generality, that
\begin{equation}\label{eq:wlog}
  \DP(C) < -\tau,
\end{equation}
\ie the \ac{GAS vector} is unfair against group $\gA$. The opposite case follows by exchanging the roles of the two groups.

Because the budget constraint ties $\Delta B = n - \Delta A$, a move has three degrees of freedom: how to split the budget between the two groups, and how many positives to claim in each. The objective \eqref{eq:objective} then reads
\begin{equation}\label{eq:objective:split}
    \DP(C + \delta) =
        \frac{a + \Delta a}{A + \Delta A}
    -
        \frac{b + \Delta b}{B + n - \Delta A} .
\end{equation}

\begin{lemma}[Problem reduction to extreme labels]\label{lem:to1D}
    For a fixed split $\Delta A$, $\DP(C + \delta)$ is strictly increasing in $\Delta a$ and strictly decreasing in $\Delta b$. Every maximizer of $\DP(C + \delta)$ on $\mathcal{X}_{n}$ satisfies $\Delta a = \Delta A$ and $\Delta b = 0$. 
\end{lemma}

\Cref{lem:to1D} (proof postponed to \Cref{app:proof:thm:principal}) relates to the attacker~\eqref{eq:attack} as follows. By \eqref{eq:wlog}, the attacker needs to raise the estimate. If no move in $\mathcal{X}_{n}$ reaches $\DP(C + \delta) \geq 0$, minimizing $|\DP(C + \delta)|$ amounts to maximizing $\DP(C + \delta)$, so $\delta^\star$ claims that every group-$\gA$ query is positive and every group-$\gB$ query negative. Otherwise, these extreme labels would overshoot above $0$, and $\delta^\star$ instead uses intermediate labels that bring $\DP(C + \delta^\star)$ as close to $0$ as possible (Algorithm~\ref{alg:attacker}).

The extreme-label move with split $\Delta A$, \ie $\delta = (\Delta A, \Delta A, 0, n - \Delta A)$, yields $\DP(C + \delta) = \varphi(\Delta A)$, where for $y \in [0, n]$
\begin{equation}\label{eq:phiX}
    \varphi(y) =
        \frac{a + y}{A + y}
    -
        \frac{b}{B + n - y}.
\end{equation}
Both denominators are positive on $[0, n]$.

The attacker's problem thus reduces from four fabricated numbers to a single variable: how to split its budget $n$ between the two groups.

\subsection{Fairwashing Conditions for a Single Adversary}
\label{sec:adversary:theory}

Fix one adversary $k$. We look for a necessary and sufficient condition for the existence of an integer split $\Delta A \in \{0, \dots, n\}$ such that $|\varphi(\Delta A)| < \tau$.

\begin{theorem}\label{thm:principal}
Assume
\begin{equation}\label{eq:hypM}
  \min(A, B) + 1 > \frac{1}{2\tau}.
\end{equation}

Then there exists $\Delta A \in \{0, \dots, n\}$ such that $|\varphi(\Delta A)| < \tau$
if and only if the following two conditions hold:
\begin{enumerate}
  \item $\min\big(\varphi(0), \varphi(n)\big) < \tau$, that is,
  \[
    \DP(C) + \min\left( \frac{n\, b}{B(B + n)},\ \frac{n\,(A - a)}{A(A + n)} \right)
    < \tau ;
  \]
  \item $\varphi(\Delta A) > -\tau$ for some $\Delta A \in \{0, \dots, n\}$.
\end{enumerate}
\end{theorem}

Condition~2 can be checked in closed form. It holds if and only if a quadratic has two roots with an integer split between them (Lemma~\ref{lem:Iplus}, Appendix~\ref{app:quadratic}). The proof is given in Appendix~\ref{app:proof:thm:principal}. Since the extreme-label moves belong to $\mathcal{X}_{n}$, these conditions are sufficient for the attacker~\eqref{eq:attack} to reach the fair band at round $t$ (Corollary~\ref{cor:attack}).

\textbf{Takeaway of Theorem~\ref{thm:principal}.} The assumption~\eqref{eq:hypM} and condition~1 are mild: with $\tau = 0.05$, they hold once $\min(A, B) \geq 10$ and $\max(A, B) \geq 9n$ (Lemma~\ref{lem:cond1}). The binding condition is condition~2: a single \ac{LAS vector} must close the gap $|\DP(C)| - \tau$ to the fair band. Hence, the closer an \ac{LLM} is to being fair, the easier it is to fairwash. This gap also shrinks over rounds, since $C$ already contains the adversary's past fabrications. Accordingly, without defense, a single adversary fairwashes every near-fair run on both datasets, but only part of the very unfair ones (Section~\ref{sec:experiments}).
\subsection{Multiple Non-Colluding Adversaries}
\label{sec:adversary:multi}

Theorem~\ref{thm:principal} describes what one adversary $k$ can reach from its own view $C$. With $K > 1$, the other adversaries are part of this view: their past fabrications are in $C$, and their \acp{LAS vector} of round $t$ are unknown to $k$, like the honest \ac{LAS vector} of that round. The theorem therefore applies to each adversary as is.

\textbf{More adversaries fairwash faster.}
Every adversary adds $n$ fabricated queries per round, all pushing the estimate towards $0$. With $K$ adversaries, fabricated mass accumulates in the \ac{GAS vector} $K$ times faster, so the gap that remains to be closed in condition~2 of Theorem~\ref{thm:principal} shrinks faster as well. Figure~\ref{fig:multiple_adversary_no_defense} shows this effect on \civil with $N = 20$ honest auditors. With honest auditors alone, the estimate stays at $\mathrm{DP}(f) \approx -0.23$. A single adversary moves it slowly, and the estimate enters the fair band only after about $80\%$ of the rounds. Two adversaries bring it into the band after about $25\%$ of the rounds, and three after about $10\%$. Once inside the band, the adversaries keep pushing, because their objective~\eqref{eq:attack} is $0$ rather than the edge of the band. As a result, all curves end close to $0$.

\subsection{Takeaway}
\label{sec:adversary:takeaway}

Fairwashing is cheap: adversaries need not query the \ac{LLM}, as each finds its best \ac{LAS vector} by simply trying every split of its budget between the two groups.
It is also effective: a single adversary suffices to make an \ac{LLM} appear fair, all the more easily as it is close to the fair band, and its influence persists since past fabrications keep counting in the estimate.
Nor do several adversaries need to coordinate: each attacking from its own view, they can together push the estimate past $0$.
Worse, each fabricated \ac{LAS vector}, taken alone, looks like an honest one from an auditor with an unusual query set, so robust aggregation cannot flag it. What does betray an adversary is its sequence of reports: to keep cancelling the honest \acp{LAS vector}, it must adapt its own every round, so they stop resembling its past ones.
Our defense exploits exactly this (\Cref{sec:defense}). %
\section{Defense}\label{sec:defense}

The defense at the core of \sys limits the influence of fabricated \acp{LAS vector} on each honest auditor's estimate (Step~5 of \Cref{alg:sys}).
We first explain why honest and adversarial increments differ, then how each increment is scored, weighted, and aggregated into the defended estimate.

Each honest auditor runs the defense locally, on the messages it receives. It needs no trusted party, no coordination with the other auditors, and no offline calibration. The defense rests on one observation. An honest auditor samples its queries from a fixed query set, so each \ac{LAS vector} is statistically consistent with its previous ones. The adversary of Section~\ref{sec:adversary} instead re-optimizes its \ac{LAS vector} every round to cancel the honest ones, so its increments keep departing from what it sent before. \sys therefore checks each source against its own history, never against the other sources. Under heterogeneity, honest sources legitimately disagree with each other, and this disagreement is what defeats robust aggregation techniques such as the median and the trimmed mean (\Cref{sec:background}).

\subsection{Honest and Adversarial Increments}\label{sec:defense-model}

Fix a source $j$, and let $c=(a,A,b,B)$ be the \ac{LAS vector} that auditor $i$ last held for it. If $j$ is honest, $c$ counts the queries that $j$ has drawn so far from its query set $\mathcal D_j$, and its next increment $\delta=(\Delta a,\Delta A,\Delta b,\Delta B)$ is drawn from the same query set.
Three statistics of $\delta$ then approximately follow binomial distributions, whose parameters $c$ estimates: 
\begin{enumerate*}[label=(\textit{\roman*})]
\item the number $\Delta A$ of group-$\gA$ queries among the $\Delta A+\Delta B$ queries of the increment, with the proportion of group~$\gA$ in $\mathcal D_j$ as parameter;
\item the number $\Delta a$ of positive decisions among these $\Delta A$ queries, with the \ac{LLM}'s positive rate on group~$\gA$ in $\mathcal D_j$ as parameter;
\item likewise, $\Delta b$ among the $\Delta B$ group-$\gB$ queries.
\end{enumerate*}
The approximation lies in treating the rates estimated from $c$ as exact, not in sampling without replacement: the increment and the history are two disjoint uniform samples of the same query set, so the variance of the difference between their rates does not depend on how much of the query set has been used. The binomial model is therefore accurate once the history is long compared with the increment, and overstates the deviation of an honest increment only in the first rounds.

The adversary of \Cref{sec:adversary} follows none of these distributions, because it changes its \ac{LAS vector} as honest ones arrive. Its first \ac{LAS vector} is neutral, with rate $1/2$ in each group. From round~2 on, it claims that every group-$\gA$ query is positive and every group-$\gB$ query negative (Lemma~\ref{lem:to1D}), so its rates jump away from its own past. It also adjusts the split of its budget to its current view of the \ac{GAS vector}, so its group proportion varies from round to round. Once the estimate approaches~0, it switches to intermediate labels, and with several adversaries, each one corrects for the overshoot of the others (Section~\ref{sec:adversary:multi}). Each of these moves appears as a deviation on at least one of the three
statistics.

\subsection{Suspicion Score}\label{sec:defense-score}

Auditor $i$ scores the increment $\delta=\maptab{i}[j].c-\mathit{prev}_j$ of $j$ each time it replaces its entry $\maptab{i}[j]$ with a newer one (Lines~17-22 of Algorithm~\ref{alg:sys}).
When gossip delays an entry, this increment covers several rounds of $j$; the score applies to increments of any size. 
\Cref{alg:suspicion} computes the suspicion score $S$ of $\delta$ against the source's \ac{LAS vector} $c=\mathit{prev}_j$.
It relies on the binomial $z$-score, computed as $z(s;m,p)=\frac{s-mp}{\sqrt{mp(1-p)}}$,
of $s$ successes in $m$ trials with probability $p$, which measures, in standard deviations, how far $s$ lies from what the binomial model predicts.
\Cref{alg:suspicion} proceeds in four steps.

\textbf{Estimate the source's parameters (Line~2).} From $c$, auditor $i$ estimates the three parameters of Section~\ref{sec:defense-model}: the proportion $\hat\pi$ of group~$\gA$ in the source's queries, and its positive rates $\hat q_{\gA}$ and $\hat q_{\gB}$ in each group. We use Krichevsky-Trofimov (KT) estimates, which add $1/2$ to each \ac{LAS vector}. They always lie in $(0,1)$, even when the history contains no positive or no negative decision, so the z-scores are always defined, and a short history with extreme rates is not treated as a certainty.

\textbf{Check the group proportion (Line~3).} $z_1$ compares the number $\Delta A$ of group-$\gA$ queries in the increment with the number predicted by the source's own query set. $s=\Delta A$ successes in $m=\Delta A+\Delta B$ trials, which equals the budget $n$ for an increment covering one round, with probability $\hat\pi$.
A large $|z_1|$ reveals a source that changes how it splits its budget between the groups.

\textbf{Check the positive rates (Lines~4-5).} $z_2$ compares the number $\Delta a$ of group-$\gA$ positives in the increment with the source's own rate. $s=\Delta a$ successes in $m=\Delta A$ trials with probability $\hat q_{\gA}$. It is defined only if the increment contains group-$\gA$ queries ($\Delta A>0$). $z_3$ does the same for group~$\gB$, with $s=\Delta b$, $m=\Delta B$ and $\hat q_{\gB}$. A large $|z_2|$ or $|z_3|$ reveals a source that changes its labels. For instance, after its neutral first \ac{LAS vector} ($\hat q_{\gA}\approx 1/2$), the adversary claims $\Delta a=\Delta A$, which gives
$z_2=\sqrt{\Delta A\,(1-\hat q_{\gA})/\hat q_{\gA}}\approx\sqrt{\Delta A},$
\ie about 12 for 150 group-$\gA$ queries.

\textbf{Combine (Line~6).} The score is the largest of the defined $|z_\ell|$, capped at $z_{\max}$:
$S=\min\bigl(\max_{\ell}|z_\ell|,\ z_{\max}\bigr).$

Taking the maximum means that a deviation on a single statistic suffices, \eg an adversary that keeps its split but changes its labels.
The cap bounds the score, so that no increment ever receives zero weight.
An honest increment that deviates from the binomial model, \eg early in the audit when its source's history is still short, keeps part of its weight.

\begin{algorithm}[t]
\caption{Suspicion score of one increment}
\small
\label{alg:suspicion}
\begin{algorithmic}[1]
\Require increment $\delta=(\Delta a,\Delta A,\Delta b,\Delta B)$ with $\Delta A+\Delta B>0$; \ac{LAS vector} $c=(a,A,b,B)$ of the same source, with $A+B>0$
\Ensure suspicion score $S\in[0,z_{\max}]$
\Function{Suspicion}{$\delta, c$}
  \State $\hat\pi \gets \frac{A+1/2}{A+B+1}$;\;
         $\hat q_{\gA} \gets \frac{a+1/2}{A+1}$;\;
         $\hat q_{\gB} \gets \frac{b+1/2}{B+1}$ \Comment{KT estimates}
  \State $z_1 \gets z(\Delta A;\ \Delta A+\Delta B,\ \hat\pi)$
         \Comment{group-$\gA$ proportion}
  \State $z_2 \gets z(\Delta a;\ \Delta A,\ \hat q_{\gA})$ \textbf{if} $\Delta A>0$
         \Comment{rate in $\gA$}
  \State $z_3 \gets z(\Delta b;\ \Delta B,\ \hat q_{\gB})$ \textbf{if} $\Delta B>0$
         \Comment{rate in $\gB$}
  \State \Return $\min\bigl(\max_\ell |z_\ell|,\ z_{\max}\bigr)$ over the
         defined $z_\ell$
\EndFunction
\end{algorithmic}
\end{algorithm} 
\subsection{Locked Weights}\label{sec:defense-weights}

An increment with score $S$ receives the weight $w=(1+S)^{-2}$, and auditor $i$ adds $w\,\delta$ to the weighted \ac{LAS vector} $W_j$ of its source (Lines~23-25 of Algorithm~\ref{alg:sys}). The weighting is soft. There is no threshold to calibrate, and an honest source with an unlucky increment is down-weighted rather than excluded. The weight is fixed once assigned and never revised. An adversary therefore cannot build up trust to rehabilitate what it sent earlier, and trust gained on earlier increments does not carry over to later ones, since each increment is weighted on its own. The first increment of each source has no history to compare with and receives $w=1$.

\subsection{Defended Estimate}\label{sec:defense-estimate}

Auditor $i$ sums the weighted \acp{LAS vector} of all sources, including its own (Line~29, \Cref{alg:sys}) as $\widehat{\DP}_i=\DP\Bigl(\sum_{j\in\mathcal V}W_j\Bigr)$,
where $W_j$ is the sum of the increments of $j$, each multiplied by its weight. Auditor $i$ weights its own \acp{LAS vector} like everyone else's. Trusting its own \acp{LAS vector} would anchor its estimate on its own skewed query set, and under heterogeneity, honest auditors would then disagree even without adversaries.

\textbf{Agreement.}
Every weight depends only on the source's own \acp{LAS vector} and on how auditor $i$ received them. Two honest auditors that received the same increments for every source therefore compute the same estimate. This holds on the complete graph, where every entry reaches every auditor in the round it is produced. On sparser graphs, gossip delays can merge several rounds into one increment for some auditors but not for others, so their estimates may differ slightly. Section~\ref{sec:experiments} combines them through their median.

\textbf{Cost.}
Each round, an auditor sends each neighbor its table: $|\mathcal V|$ signed entries, each a \ac{LAS vector} and a round number. Scoring takes at most $|\mathcal V|$ calls to Algorithm~\ref{alg:suspicion}, each computing three z-scores. The defense shares nothing beyond the \acp{LAS vector} that the protocol already exchanges. %
\section{Experimental Evaluation}
\label{sec:experiments}

We implement \sys\footnote{Source code: \url{https://anonymous.4open.science/r/auditopus-C04A}.} and perform an extensive experimental evaluation of our system and baselines.
Our experiments answer the following two questions:
\begin{enumerate}
  \item \textbf{Effectiveness (RQ1).} How accurately do \sys and baselines preserve the audit's verdict and its \ac{DP} estimate under attack as the proportion of adversaries grows (\Cref{sec:eval_rq1})?
  \item \textbf{Sensitivity (RQ2).} How do the number of honest auditors and level of data heterogeneity affect the robustness of \sys and baselines (\Cref{sec:eval_rq2})?
\end{enumerate}

We provide additional experiments in \Cref{app:experiments} that assess how our results generalize to a second dataset (\Cref{app:exp_civil}), and compare against additional variants of the median and trimmed-mean baselines (\Cref{app:detailed_baselines}).
We further validate our design choices, including query sampling without replacement (\Cref{app:exp_sample}) and the query budget per round (\Cref{app:budget}).

\begin{table*}[t]
\centering
\small
\caption{Summary of the six audit cases and the true \ac{DP} of the audited \acp{LLM}. \emph{Case}: very unfair, moderate or near-fair, depending on the magnitude of the true DP; \emph{Population}: the inputs the auditors query, \ie, the comments mentioning group~$\mathsf{A}$ or $\mathsf{B}$ (\civil) or the biographies in the given sector (\biasbios); \emph{Positive}: the prediction defined as positive; for \biasbios, the multi-class occupation classifier is turned into a one-vs-rest classifier for the target occupation (\Cref{app:exp_details}). $|\mathsf{A}|$, $|\mathsf{B}|$: number of inputs per group; \emph{Rate}: proportion of positive predictions per group; True DP $=$ Rate~$\mathsf{A}$ $-$ Rate~$\mathsf{B}$ over the whole population.} 
\label{tab:6_cases}
\setlength{\tabcolsep}{5pt}
\begin{tabular}{l l l l l r r r r r}
\toprule
Case & Attribute & $\mathsf{A}$ vs.\ $\mathsf{B}$ & Population & Positive & $|\mathsf{A}|$ & $|\mathsf{B}|$ & Rate $\mathsf{A}$ & Rate $\mathsf{B}$ & True DP \\
\midrule
\multicolumn{10}{l}{\textbf{\civil} \quad \qwen, zero-shot toxicity detection} \\
Very unfair & Religion & Christian vs.\ Muslim & comments & toxic & 39{,}688 & 17{,}622 & 0.208 & 0.436 & $-0.228$ \\
Moderate    & Religion & Christian vs.\ Jewish & comments & toxic & 39{,}688 &  4{,}763 & 0.208 & 0.398 & $-0.190$ \\
Near-fair   & Race     & Black vs.\ White      & comments & toxic &  9{,}385 & 16{,}926 & 0.476 & 0.537 & $-0.060$ \\
\midrule
\multicolumn{10}{l}{\textbf{\biasbios} \quad \llama, zero-shot occupation prediction (one-vs-rest)} \\
Very unfair & Gender & male vs.\ female & Healthcare & \emph{nurse}     & 45{,}742 & 48{,}985 & 0.028 & 0.300 & $-0.272$ \\
Moderate    & Gender & female vs.\ male & Healthcare & \emph{physician} & 48{,}985 & 45{,}742 & 0.423 & 0.571 & $-0.148$ \\
Near-fair   & Gender & male vs.\ female & Education  & \emph{teacher}   & 71{,}261 & 63{,}018 & 0.035 & 0.098 & $-0.063$ \\
\bottomrule
\end{tabular}
\end{table*} 
\subsection{Experimental Setup}
\label{sec:exp_setup}
We next describe our experimental setup, including dataset and baseline descriptions, and protocol parameters.

\textbf{Setup.}
We audit two zero-shot \ac{LLM} classifiers: \qwen~\cite{qwen25} for toxicity on \civil~\cite{civilcomment}, and \llama~\cite{grattafiori2024llama} for occupation (one-vs-rest) on \biasbios~\cite{de2019bias}.
For each dataset, we pick three unfair cases, which we call \emph{very unfair}, \emph{moderate} and \emph{near-fair}, with a large, medium and small true \ac{DP}, respectively (\Cref{tab:6_cases}).
Auditors hold non-IID query sets drawn from a Dirichlet distribution ($\alpha = 1$ unless stated otherwise).
Further experimental details are given in \Cref{app:exp_details}.
We fix $N=20$ honest auditors, unless stated otherwise.

\textbf{Protocol parameters.}
The $N + K$ auditors, attackers included, communicate over an Erd\H{o}s--R\'enyi graph with an expected average degree of $4$. We redraw the graph until it is connected and keep it fixed during the audit, so attackers occupy random positions in the network.
In every round, each auditor issues $n = 300$ queries to the \ac{LLM}.
We run each audit for at most $T=150$ rounds.
At that point, the honest auditors have jointly queried every input in the population exactly once, so the DP computed from the honest auditors' \acp{LAS vector} alone equals the true DP. Any deviation from it is therefore caused by the attackers.
We report results for querying with replacement in \Cref{app:exp_sample}.
We use a fair band of $\tau = 0.05$ across all experiments.

\textbf{Attack and defense.}
We vary the number of attackers $K \in \{0, \dots, 19\}$.
Each attacker runs \Cref{alg:attacker} independently.
We set $z_{\max}=8$ in all our experiments (\Cref{alg:suspicion}).

\textbf{Baselines.}
We compare \sys with three baselines%
:
(i)~\emph{no defense}, which sums all entries into the \ac{GAS vector};
(ii)~the \emph{median}~\cite{yin2018byzantine}; and
(iii)~the \emph{trimmed mean}~\cite{yin2018byzantine}, which drops the $\lfloor \beta m \rfloor$ smallest and largest of the $m$ values before averaging.
Baselines (ii) and (iii) aggregate the positive rate of each group separately (\ie $a/A$ and $b/B$), and take \ac{DP} as the difference of the two results.
We use $\beta = 20\%$. \Cref{app:detailed_baselines} reports results for $\beta \in \{10, 20, 30, 40\}\%$ and for both robust aggregators applied directly to each source's local \ac{DP}.

\textbf{Metrics.}
For our experiments, we combine the auditors' final estimates into a single verdict as follows.
At the end of the audit, every auditor sends its \ac{DP} estimate (Line~29 of \Cref{alg:sys}) to a single party, \eg, a regulator.
An adversary always sends $0$, to make the \ac{LLM} appear fair.
The global \ac{DP} estimate $\widehat{\DP}$ is the median of the $N + K$ \ac{DP} estimates.
Every \ac{LLM} we audit in our experiments is unfair, and we order the groups so that $\DP(f) < -\tau$ (\Cref{tab:6_cases}).
Hence any estimate in the fair band is a wrong verdict, and any positive error moves the estimate towards the band.
We focus on two metrics:
\begin{enumerate*}[label=(\textit{\roman*})]
\item \emph{\acf{ASR}}, the proportion of runs where $|\widehat{\DP}| \le \tau$, \ie, where an unfair \ac{LLM} is declared fair;
\item the \emph{DP estimation error} $\widehat{\DP} - \DP(f)$, which flips the verdict once it exceeds $|\DP(f)| - \tau$. An estimate pushed to $0$ has error $|\DP(f)|$.
\end{enumerate*}

For every configuration, we run $20$ seeds, where a seed fixes the data partition, the communication graph and the query order.
We average all results and present error bars showing standard deviation where applicable.

\begin{table*}[t]
\centering
\small
\caption{The \acf{ASR} and mean absolute \ac{DP} estimation error on the \biasbios dataset, over every run with $1$--$19$
adversaries ($5$--$49\%$ of all auditors; $380$ runs per case). We provide similar results for \civil in
\Cref{tab:rq1-cc-table}.}
\label{tab:rq1-bios-table}
\setlength{\tabcolsep}{5pt}

\begin{tabular}{lrrrrrrrr}
\toprule
 & \multicolumn{4}{c}{Attack success rate (\%)} & \multicolumn{4}{c}{Average DP estimation error} \\
\cmidrule(lr){2-5}\cmidrule(lr){6-9}
\ac{LLM} fairness & No-defense & Median & Trimmed mean & \sys{} & No-defense & Median & Trimmed mean & \sys{} \\
\midrule
Very unfair & 95 & 5 & \textbf{0} & \textbf{0} & 0.267 & 0.137 & 0.079 & \textbf{0.060} \\
Moderate & 97 & 46 & 71 & \textbf{0} & 0.144 & 0.189 & 0.155 & \textbf{0.034} \\
Near-fair & 100 & 20 & 53 & \textbf{9} & 0.062 & 0.079 & 0.037 & \textbf{0.006} \\
\midrule
All three cases & 97 & 24 & 41 & \textbf{3} & 0.158 & 0.135 & 0.090 & \textbf{0.034} \\
\bottomrule
\end{tabular}

\end{table*}

\begin{figure*}[t]
  \centering
    \includegraphics[width=.8\textwidth]{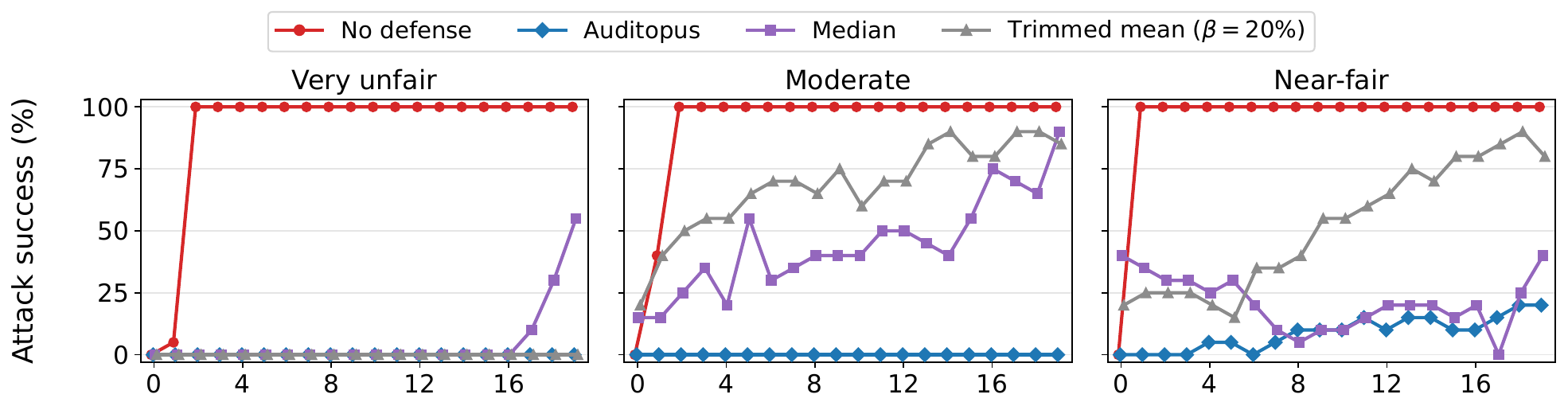} 
    
    \vspace{0.5em}
    \includegraphics[width=.8\textwidth]{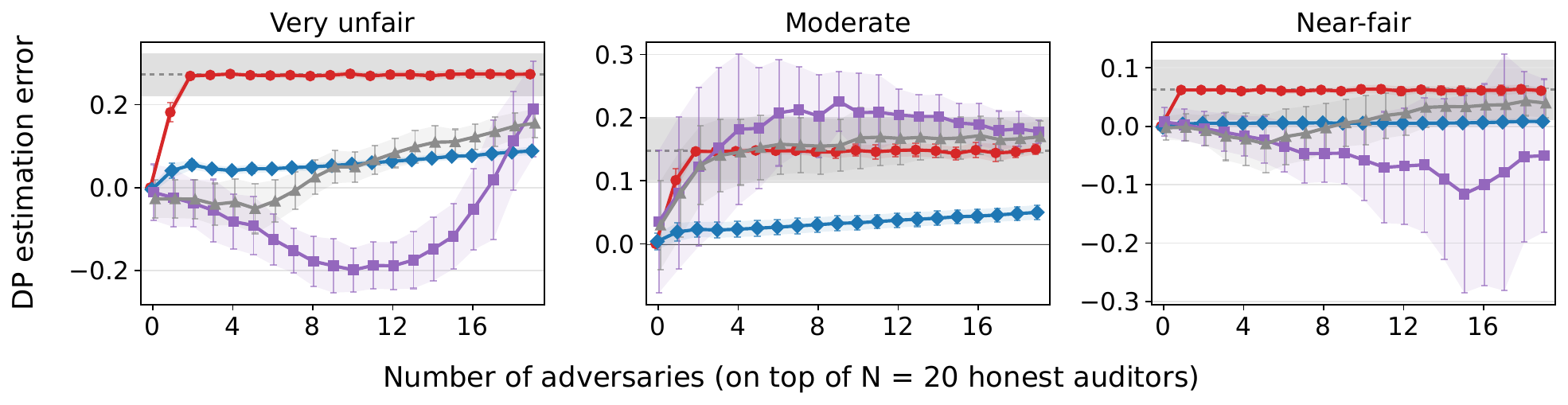}
    \caption{The attack success rate (top row) and \ac{DP} estimation error (in the bottom row) for \biasbios, while varying the number of adversaries $K$ from $0$ to $19$. We consider a very unfair (left column), moderate (middle column) and near-fair \ac{LLM} (right column). The dashed horizontal line (in the bottom row) corresponds to a \ac{DP} estimate of 0. Results for the \civil dataset are provided in \Cref{app:exp_civil}.} 
  
  \label{fig:rq1-bios}
\end{figure*}

\subsection{RQ1: Effectiveness of \sys and baselines}
\label{sec:eval_rq1}

We first quantify the robustness of \sys and baselines on the \biasbios dataset (we provide results for \civil in \Cref{app:exp_civil}).
\Cref{tab:rq1-bios-table} reports the mean \ac{ASR} and mean absolute \ac{DP} estimation error of \sys and baselines, over every run with $K = 1$ to $19$ adversaries.
Its first three rows correspond to the three \biasbios cases of \Cref{tab:6_cases}, from the very unfair to the near-fair \ac{LLM}, and the last row combines them.
This last row highlights the effectiveness of \sys.
Without any defense, the attack succeeds in $97\%$ of the audits, and the median and the trimmed mean let it through in $24\%$ and $41\%$ of the audits, respectively.
In stark contrast, \sys almost eliminates the attack, which succeeds in only $3\%$ of the audits.
\sys also substantially improves accuracy, reducing the mean absolute error from $0.158$ to $0.034$ ($-78\%$), and by $75\%$ and $62\%$ relative to the median and the trimmed mean.
These results demonstrate that \sys remains robust against adversaries and consistently outperforms all baselines.

\Cref{fig:rq1-bios} breaks down these results by number of adversaries $K$, up to $K=19$ (top: \ac{ASR}, bottom: estimation error).
Without a defense, fairwashing is remarkably easy: a single adversary already succeeds in $100\%$, $40\%$, and $5\%$ of the runs for the near-fair, moderate, and very unfair \ac{LLM}, respectively, \ie more often the closer the \ac{LLM} is to the fair band, as \Cref{thm:principal} predicts.
Two adversaries suffice to succeed in every run, driving the estimate all the way to $0$ (dashed line) and making even the very unfair \ac{LLM} appear perfectly fair.
Robust aggregation does not solve this.
The median and the trimmed mean aggregate the auditors' individual positive rates, which differ across heterogeneous query sets, so their estimate is biased even before any attack: with $K=0$, they already declare the moderate \ac{LLM} fair in $15\%$ and $20\%$ of the runs, and $46\%$ and $71\%$ under attack.
The median's low \ac{ASR} on the very unfair \ac{LLM} ($5\%$) is no sign of robustness either: up to $K=16$, the adversaries push its estimate away from the band, so that it overstates the disparity, and beyond, it breaks down ($55\%$ at $K=19$).
We analyze both baselines in more detail in \Cref{app:detailed_baselines}.

\sys, in contrast, keeps the \ac{ASR} at $0\%$ on the very unfair and moderate \acp{LLM}, even when $K=19$ of the $39$ auditors ($49\%$) are adversarial.
Its estimate does drift slowly towards the fair band as $K$ grows, since \sys down-weights the fabricated \acp{LAS vector} rather than discarding them entirely, but this drift remains modest: even at $K=19$, the error reaches only $+0.088$ on the very unfair \ac{LLM}, \ie $40\%$ of the $0.222$ gap between its true \ac{DP} and the fair band, and about half of that gap on the moderate \ac{LLM}.
The near-fair \ac{LLM} is the hardest case since, as seen above, a small shift suffices to flip its verdict. The \ac{ASR} of \sys in this case remains $0\%$ up to $K=3$ but reaches $20\%$ at $K=19$.
Even then, \sys keeps the estimation error below $+0.008$ for every $K$, about $8$x lower than the $+0.061$ without defense. In answer to RQ1, \sys remains effective across all fairness levels in \biasbios, even when nearly half of the auditors ($49\%$) attempt to fairwash the \ac{LLM}.

\begin{figure*}[t]
  \centering
  \includegraphics[width=\textwidth]{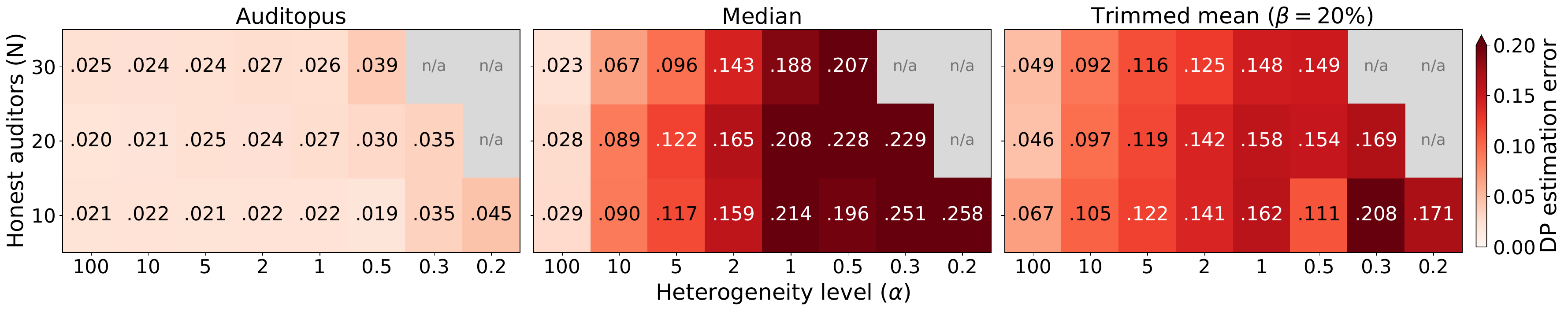}
  \caption{The \ac{DP} estimation error on the \biasbios dataset, for \sys and baselines, on the moderate \ac{LLM}. We vary the number of honest auditors ($N$) and heterogeneity level ($\alpha$). Lower is better.}
  \label{fig:rq2}
\end{figure*}

\subsection{RQ2: Sensitivity to honest auditors and heterogeneity level}
\label{sec:eval_rq2}
We next explore how the number of honest auditors $N$ and the heterogeneity of their query sets affect the robustness of \sys and baselines.
In \Cref{fig:rq2}, we report the \ac{DP} estimation error on the moderate \ac{LLM} of \biasbios for $N \in \{10, 20, 30\}$ and heterogeneity levels $\alpha$ ranging from $100$ (near-homogeneous query sets) to $0.2$ (highly heterogeneous), with the number of adversaries $K$ fixed at $30\%$ of $N$. %
Some cells in \Cref{fig:rq2} are left out since, for some low values of $\alpha$, some seeds have no feasible data split in which every auditor holds at least 30 queries of each group.

\Cref{fig:rq2} shows that \sys is rather insensitive to changing the number of honest auditors and heterogeneity level, with values ranging from 0.019 to 0.045, thus changing by at most $2.1$x across $N$ at a fixed $\alpha$.
At the same time, both robust aggregation baselines break down as the query sets become more heterogeneous, \ie as $\alpha$ decreases.
With near-homogeneous query sets ($\alpha = 100$), the median is as accurate as \sys, with an error less than $0.03$ for every $N$.
As $\alpha$ decreases, however, the errors of both the median and the trimmed mean grow steadily: for $N = 10$ and $\alpha = 0.2$, they reach $0.26$ and $0.17$, respectively.
This is because robust aggregators assume that the honest values are concentrated. Adversaries can then only shift the aggregate from one honest positive rate to another, which barely changes the estimate when these rates are close to each other. Heterogeneity spreads the honest rates apart, so the same shift moves the estimate much further.
\sys, which judges each source only against its own history, does not rely on this assumption.

Answer to RQ2: \sys is insensitive to both the number of honest auditors and the heterogeneity of their query sets, with a \ac{DP} estimation error between $0.02$ and $0.05$ in every setting, whereas the error of the median and the trimmed mean grows with heterogeneity, up to $0.25$ and $0.21$, respectively. %
\section{Related Work}
\label{sec:related_work}

\textbf{Manipulation of black-box audits.}
Black-box fairness auditing is usually framed as estimating a metric from the model's answers to a limited number of queries~\cite{rastegarpanah2021auditing}.
Such audits are easy to game when the platform is adversarial.
For example, it can justify unfair decisions with fair-looking explanations~\cite{shahin2022washing}, present a biased sample as evidence of fairness~\cite{fukuchi2020faking, lafargue2025exposing}, or answer audit queries differently from regular traffic~\cite{aivodji2019fairwashing}.
Defenses check the answers against the auditor's prior knowledge~\cite{garcia2025robust} or against an independent trusted source~\cite{bourree2023leveraging}.
Others hide the audit set inside a larger candidate set using private information retrieval~\cite{godinot2026manipulation}, or replace querying with cryptographic proofs of fairness~\cite{shamsabadi2022confidential,yadav2024fairproof}.
All these works treat the platform as the adversary and the auditor as trusted.
\sys, however, addresses a complementary threat where the platform acts through a colluding auditor.

\textbf{Collaborative auditing and poisoned aggregates.}
Alongside decentralized methods for \emph{training} fair \ac{ML} models~\cite{biswas2025fair}, a recent line of work explore how several parties can collaborate to \emph{audit} a model for fairness, though not yet what happens when some of them cannot be trusted.
Auditors can aggregate their query budgets and coverage~\cite{de2024fairness}. In collaborative active learning, \textsc{FAIR}~\cite{xu2023fair} coordinates how several parties select their queries so that each one benefits from the collaboration (individual rationality) and the gains are shared equitably. \textsc{FaaS}~\cite{toreini2024fairness} lets many parties compute a fairness metric jointly, with zero-knowledge proofs that the submitted cryptograms are well formed.
Such proofs show that the computation is correct, not that an auditor's inputs reflect the model's real answers, so such works assume honest auditors.
The closest threat studied so far is poisoning of distributed statistics. Fake users of a local differential privacy protocol can craft reports that shift estimated frequencies at will~\cite{cao2021data}.
Our attacker plays the same game against an aggregated fairness metric, but \sys exploits what local differential privacy protocols lack. Each auditor reports over many rounds, which gives honest auditors a history to check every new \ac{LAS vector} against.

\textbf{Byzantine-robust aggregation and gossip.}
Robust aggregators like the median and  %
trimmed mean~\cite{yin2018byzantine}
discard inputs lying far from the majority.
Robust gossip extends this to peer-to-peer networks~\cite{he2022byzantine, gaucher2025unified} but with weakened guarantees under data heterogeneity~\cite{karimireddy2022bucketing, allouah2023robust}.
Robust gossip also typically uses the receiving node's own value as its trusted reference~\cite{he2022byzantine, gaucher2025unified}.
\sys keeps every auditor's \acp{LAS vector} and never compares auditors with each other, so honest but atypical auditors are not penalized.

\textbf{History-based detection.}
Robust distributed learning has also exploited participants' past behavior, either through momentum~\cite{karimireddy2021learning} or by predicting each client's next model update and flagging clients whose updates keep deviating from it~\cite{zhang2022fldetector}.
\sys applies this idea to \acp{LAS vector}, where honest behavior is known exactly. An honest auditor samples from a fixed query set, so each new increment follows a binomial distribution given its own past.
Unlike these detectors, \sys therefore needs no learned predictor and no trusted server, and it down-weights \acp{LAS vector} with locked weights instead of removing participants.
Attackers that stay consistent across rounds can evade history-based detectors~\cite{xie2025model}, but in our setting such an attacker gives up the per-round adaptivity it needs to cancel the honest \acp{LAS vector}. %
\section{Conclusion}
\label{sec:conclusion}
We studied collaborative fairness auditing of black-box \acp{LLM}, where auditors with small, heterogeneous query sets estimate demographic parity together by exchanging only \aclp{LAS vector}.
We showed, theoretically and empirically, that \aclp{GAS vector} are accurate under heterogeneity but fragile. A single adversary can fabricate its \acp{LAS vector} to make an unfair \ac{LLM} appear fair, and standard robust aggregation fails because honest auditors with atypical data deviate from the majority as well.
\sys is a novel round-based decentralized auditing protocol in which each honest auditor down-weights sources whose \acp{LAS vector} are statistically inconsistent with their own history, without a trusted party and without penalizing honest but atypical auditors. Across two datasets and six audit cases, \sys reduces attack success rate and error in estimation compared against the \textit{no defense} and robust aggregation baselines.

\section{Acknowledgments}
\label{sec:acknowledgments}
This work has been funded by the Swiss National Science Foundation, under the project “FRIDAY: Frugal, PrivacyAware and Practical Decentralized Learning”, SNSF proposal No. 10.001.796.  
\section*{Open Science}
We release all artifacts needed to reproduce our results in an anonymized repository.
It contains the source code of \sys and its baselines, the datasets and prompts used in our experiments, and the scripts and configuration files to run every experiment in \Cref{sec:experiments} and \Cref{app:experiments}.
It also includes the scripts that generate all figures and tables from the raw experimental outputs.
The evaluated models are publicly available through their original providers.

\section*{LLM usage considerations}
LLMs were used for editorial purposes in this manuscript, and all outputs were inspected by the authors to ensure accuracy and originality.
We also used \acp{LLM} as coding assistants; we manually reviewed and tested all generated code for correctness and accuracy.
We did not use generative AI to propose hypotheses, design the methodology or experiments, or interpret results.
The authors take full responsibility for all content of this paper.

LLMs are also the subject of our study since \sys audits the fairness of \acp{LLM}, so our experiments necessarily query them.
To support reproducibility, we use open-weight models, and document the exact configurations and parameters.
To limit the environmental footprint, we query each model only once per prompt and store the responses.
All experiments then reuse these cached responses rather than repeating inference, so the total LLM compute is independent of the number of experiments we run.
We include these cached responses in our artifact, so our results can be reproduced without re-querying the models.
All experiments ran on our compute cluster, totaling approximately \num{100} GPU-hours.

\section*{Ethical Considerations}
This work involves no human subjects and no personal data.
All used datasets and pre-trained LLMs are publicly available.
Our goal is to make fairness audits of LLMs more trustworthy.
We study adversarial behavior to design defenses against it.
We do not introduce new attack capabilities against deployed systems or their users.
We note that no auditing procedure, including ours, can certify with absolute certainty that a model is fair.
Audit outcomes depend on the fairness metrics and demographic groups considered, which do not capture every form of harm. 

\bibliographystyle{plain}
\bibliography{main.bib}

\appendices

\crefalias{section}{appendix}
\crefalias{subsection}{appendix}
\crefalias{subsubsection}{appendix}

\onecolumn
\section{Notations}
\label{app:notations}

\begin{table*}[h]
\centering
\small
\caption{Notations used throughout the paper.}
\label{tab:notations}
\begin{tabularx}{\textwidth}{@{}l X l@{}}
\toprule
Symbol & Meaning & Defined in \\
\midrule
\multicolumn{3}{@{}l}{\textit{Protocol}} \\
$f$ & black-box \ac{LLM} under audit & Sec.~\ref{sec:bb_prelim} \\
$\gA$, $\gB$ & the two demographic groups & Sec.~\ref{sec:bb_prelim} \\
$\tau$ & half-width of the fair band $[-\tau, \tau]$ & Sec.~\ref{sec:bb_prelim} \\
$N$, $K$ & number of honest auditors and of adversaries & Sec.~\ref{sec:system_model} \\
$\mathcal{V}$ & set of all auditors, $|\mathcal{V}| = N + K$ & Sec.~\ref{sec:system_model} \\
$\mathcal{G}$, $\mathcal{N}(i)$ & communication graph (static, connected); neighbors of auditor $i$ & Sec.~\ref{sec:system_model} \\
$\mathcal{D}_i$ & private query set of auditor $i$ & Sec.~\ref{sec:colab_auditing_prelim} \\
$T$, $t$ & number of rounds; round index & Sec.~\ref{sec:system_model} \\
$n$ & per-round query budget of every auditor & Sec.~\ref{sec:protocol} \\
$Q$ & queries sampled by an auditor in one round, $|Q| = n$ & Sec.~\ref{sec:protocol} \\
$\maptab{i}$ & table of auditor $i$; $\maptab{i}[j] = (c, r)$: latest \ac{LAS vector} of $j$ known by $i$, and the round $r$ in which $j$ produced it & Sec.~\ref{sec:protocol} \\
\midrule
\multicolumn{3}{@{}l}{\textit{Statistics}} \\
$A$, $B$ & number of queries in group $\gA$, $\gB$ & Sec.~\ref{sec:colab_auditing_prelim} \\
$a$, $b$ & number of positive decisions among them & Sec.~\ref{sec:colab_auditing_prelim} \\
$c_i = (a_i, A_i, b_i, B_i)$ & \ac{LAS vector} of auditor $i$, cumulative over rounds & Sec.~\ref{sec:colab_auditing_prelim} \\
$C = (a, A, b, B)$ & \ac{GAS vector}, sum of \acp{LAS vector}; in Sec.~\ref{sec:adversary}, the attacker's view of it & Eq.~\eqref{eq:pooled_dp}, \eqref{eq:history} \\
$\delta = (\Delta a, \Delta A, \Delta b, \Delta B)$ & increment of a LAS vector: fabricated in one round by an adversary; for honest auditor $i$, $\maptab{i}[j].c - \mathrm{prev}_j$, which spans several rounds under gossip delays & Sec.~\ref{sec:protocol}, ~\ref{sec:adversary:threat} \\
$\alpha$ & concentration of the Dirichlet partition (lower = more heterogeneous) & Sec.~\ref{sec:colab_auditing_prelim} \\
\midrule
\multicolumn{3}{@{}l}{\textit{Demographic parity}} \\
$\DP(f)$ & true DP of $f$ on the population, equal to the DP of the honest auditors' \ac{GAS vector} & Eq.~\eqref{eq:estimand}, Sec.~\ref{sec:system_model} \\
$\DP(c)$ & $a/A - b/B$, for any vector $c = (a, A, b, B)$ with $A, B > 0$ & Eq.~\eqref{eq:local_dp} \\
$\DP(C)$ & estimate from the \ac{GAS vector} & Eq.~\eqref{eq:pooled_dp} \\
$\widehat{\DP}_i$ & defended estimate of honest auditor $i$ & Sec.~\ref{sec:defense-estimate} \\
$\widehat{\DP}$ & median of the final estimates of all $N + K$ auditors & Sec.~\ref{sec:experiments} \\\midrule
\multicolumn{3}{@{}l}{\textit{Adversary}} \\
$\mathcal{X}_n$ & admissible increments & Eq.~\eqref{eq:feasible} \\
$\delta^\star$ & best response of the adversary & Eq.~\eqref{eq:attack} \\
$y$ & budget split to group $\gA$ (real variable; its integer values are the splits $\Delta A$) & Sec.~\ref{sec:adversary:move} \\
$\varphi(y)$ & attacker's objective under extreme labels & Eq.~\eqref{eq:phiX} \\
\midrule
\multicolumn{3}{@{}l}{\textit{Defense}} \\
$\mathit{prev}_j$ & \ac{LAS vector} of $j$ held by $i$ at the end of the previous round & Sec.~\ref{sec:protocol} \\
$z(s; m, p)$ & binomial z-score of $s$ successes in $m$ trials with probability $p$ & Sec.~\ref{sec:defense-score} \\
$\hat\pi$, $\hat q_{\gA}$, $\hat q_{\gB}$ & KT estimates of a source's proportion of $\gA$ and of its positive rates & Alg.~\ref{alg:suspicion} \\
$z_1$, $z_2$, $z_3$ & z-scores of the group proportion and of the rates in $\gA$ and $\gB$ & Alg.~\ref{alg:suspicion} \\
$S$, $z_{\max}$ & suspicion score of an increment, $S \in [0, z_{\max}]$; score cap & Alg.~\ref{alg:suspicion} \\
$w$ & locked weight of an increment: $1$ for a source's first increment, $(1 + S)^{-2}$ otherwise & Sec.~\ref{sec:defense-weights} \\
$W_j$ & weighted \ac{LAS vector} of $j$ held by $i$: $j$'s increments, each multiplied by its weight (equals $j$'s \ac{LAS vector} when all weights are $1$) & Sec.~\ref{sec:defense-weights} \\
\midrule
\multicolumn{3}{@{}l}{\textit{Experiments}} \\
$\beta$ & trimming fraction of the trimmed mean & Sec.~\ref{sec:experiments} \\
\midrule\multicolumn{3}{@{}l}{\textit{Analysis (Appendix~\ref{app:proof:thm:principal})}} \\
$I_+$, $I_-$ & integer splits with $\varphi > -\tau$, resp.\ $\varphi < \tau$ & App.~\ref{app:proof:thm:principal} \\
$\Gamma$ & quadratic $-(1 + \tau) y^2 + \gamma_1 y + \gamma_0$, of the sign of $\varphi + \tau$ & Eq.~\eqref{eq:phitau} \\
$\gamma_1$, $\gamma_0$ & coefficient of $y$ and constant term of $\Gamma$ & App.~\ref{app:quadratic} \\
$\Lambda$ & discriminant of $\Gamma$ & App.~\ref{app:quadratic} \\
$\rho_1 < \rho_2$ & roots of $\Gamma$ & App.~\ref{app:quadratic} \\
$y^-$, $y^+$ & smallest and largest integer splits in $(\rho_1, \rho_2)\, \cap [0, n]$; $I_+ = \{y^-, \dots, y^+\}$ & Lemma~\ref{lem:Iplus} \\
\bottomrule
\end{tabularx}
\end{table*}
\clearpage 
\newpage
\section{Proofs of Section~\ref{sec:adversary}}
\label{app:proof:thm:principal}

Throughout, $C = (a, A, b, B)$ is the attacker's \ac{GAS vector}, with $A > 0$ and $B > 0$ (Section~\ref{sec:adversary:move}). Every \ac{LAS vector}, honest or fabricated, counts at most as many positives as queries in each group (for fabricated ones, by definition of $\mathcal{X}_{n}$, Section~\ref{as:valid}), so $0 \leq a \leq A$ and $0 \leq b \leq B$. Recall from~\eqref{eq:phiX} that, for $y \in [0, n]$,
\[
  \varphi(y) = \frac{a + y}{A + y} - \frac{b}{B + n - y},
\]
where both denominators are positive. The integer values of $y$ are the splits $\Delta A \in \{0, \dots, n\}$ of Section~\ref{sec:adversary:move}.

\paragraph{Roadmap}
Theorem~\ref{thm:principal} asks for an integer $y \in \{0, \dots, n\}$ with $-\tau < \varphi(y) < \tau$. We treat the two inequalities separately, through the sets
\[
  I_+ = \{ y \in \{0, \dots, n\} : \varphi(y) > -\tau \},
  \qquad
  I_- = \{ y \in \{0, \dots, n\} : \varphi(y) < \tau \},
\]
so that the theorem asks whether $I_+ \cap I_- \neq \emptyset$. Appendix~\ref{app:quadratic} describes $I_+$ exactly, through a quadratic. Appendix~\ref{app:concave} shows that, by concavity, $I_-$ is nonempty if and only if it contains $0$ or $n$. Appendix~\ref{app:jumps} shows that $\varphi$ moves by small steps, so that the sequence $\big(\varphi(y)\big)_y$ cannot jump over the fair band $(-\tau, \tau)$. Appendix~\ref{app:proof:main} combines the three. We first prove Lemma~\ref{lem:to1D}, which reduces the attacker's problem to $\varphi$.

\subsection{Proof of Lemma~\ref{lem:to1D}}

Fix a split $\Delta A \in \{0, \dots, n\}$, so that $\Delta B = n - \Delta A$. By~\eqref{eq:objective:split}, $\DP(C + \delta)$ is affine in $\Delta a$ and in $\Delta b$. Increasing $\Delta a$ by one, all else fixed, changes it by
\[
  \frac{1}{A + \Delta A} > 0,
  \qquad\text{and increasing } \Delta b \text{ by one changes it by}\qquad
  -\frac{1}{B + n - \Delta A} < 0 .
\]
Hence $\DP(C + \delta)$ is strictly increasing in $\Delta a$ and strictly decreasing in $\Delta b$. Let $\delta = (\Delta a, \Delta A, \Delta b, \Delta B) \in \mathcal{X}_{n}$. If $\Delta a < \Delta A$, raising $\Delta a$ by one keeps $\delta$ in $\mathcal{X}_{n}$ and strictly increases $\DP(C + \delta)$; if $\Delta b > 0$, lowering $\Delta b$ by one does the same. So every maximizer satisfies $\Delta a = \Delta A$ and $\Delta b = 0$. Such a move is $\delta = (\Delta A, \Delta A, 0, n - \Delta A)$, and by~\eqref{eq:phiX} it satisfies $\DP(C + \delta) = \varphi(\Delta A)$, hence
\[
  \max_{\delta \in \mathcal{X}_{n}} \DP(C + \delta) = \max_{y \in \{0, \dots, n\}} \varphi(y). \tag*{$\square$}
\]

\subsection{Where $\varphi > -\tau$, the quadratic $\Gamma$}
\label{app:quadratic}

\paragraph{Reduction to a sign}
To locate the splits $y$ with $\varphi(y) > -\tau$, we reduce $\varphi(y) + \tau$ to the common denominator $(A + y)(B + n - y)$. Multiplying the first term of $\varphi$ by $B + n - y$, the second by $A + y$, and $\tau$ by both, we get
\begin{equation}\label{eq:phitau}
  \varphi(y) + \tau = \frac{\Gamma(y)}{(A + y)(B + n - y)},
  \qquad
  \Gamma(y) = (a + y)(B + n - y) - b\,(A + y) + \tau\,(A + y)(B + n - y).
\end{equation}
The denominator is positive on $[0, n]$, so $\varphi(y) + \tau$ has the sign of $\Gamma(y)$:
\begin{equation}\label{eq:sign}
  \varphi(y) > -\tau \iff \Gamma(y) > 0 .
\end{equation}

\paragraph{Expanding $\Gamma$}
The three terms of $\Gamma(y)$ expand as
\begin{align*}
  (a + y)(B + n - y) &= -y^2 + (B + n - a)\, y + a\,(B + n), \\
  -b\,(A + y) &= -b\, y - b\,A, \\
  \tau\,(A + y)(B + n - y) &= -\tau\, y^2 + \tau\,(B + n - A)\, y + \tau\, A\,(B + n).
\end{align*}
Collecting the coefficients of $y^2$, of $y$ and the constant terms shows that $\Gamma$ is a quadratic in $y$:
\[
  \Gamma(y) = -(1 + \tau)\, y^2 + \gamma_1\, y + \gamma_0,
\]
where
\begin{align*}
  \gamma_1 &= (B + n - a - b) + \tau\,(B + n - A), \\
  \gamma_0 &= a\,(B + n) - b\,A + \tau\, A\,(B + n).
\end{align*}

\paragraph{Roots}
The leading coefficient of $\Gamma$ is $-(1 + \tau) < 0$, so $\Gamma$ is positive exactly between its roots, if it has two distinct ones. Its discriminant is
\[
  \Lambda = \gamma_1^2 + 4(1 + \tau)\gamma_0 .
\]
If $\Lambda \leq 0$, $\Gamma$ is positive nowhere. If $\Lambda > 0$, the quadratic formula gives the two roots
\[
  \rho_1 = \frac{\gamma_1 - \sqrt{\Lambda}}{2(1 + \tau)}
  \;<\;
  \rho_2 = \frac{\gamma_1 + \sqrt{\Lambda}}{2(1 + \tau)},
\]
and $\Gamma(y) > 0$ if and only if $y \in\ (\rho_1, \rho_2)$.

\paragraph{Integer splits}
The splits are the integers of $[0, n]$, so we need the integers of $(\rho_1, \rho_2)\, \cap [0, n]$. The smallest integer strictly greater than $\rho_1$ is $\lfloor \rho_1 \rfloor + 1$, and the largest integer strictly smaller than $\rho_2$ is $\lceil \rho_2 \rceil - 1$. Clipping them to $[0, n]$, we set
\[
  y^- = \max\big(0,\ \lfloor \rho_1 \rfloor + 1\big),
  \qquad
  y^+ = \min\big(n,\ \lceil \rho_2 \rceil - 1\big),
\]
with the convention that $\{y^-, \dots, y^+\} = \emptyset$ if $y^- > y^+$. Combining with~\eqref{eq:sign}, we have shown:

\begin{lemma}[Where $\varphi > -\tau$]\label{lem:Iplus}
$I_+ = \{y^-, \dots, y^+\}$ if $\Lambda > 0$, and $I_+ = \emptyset$ otherwise. In particular, $I_+$ is a set of consecutive integers, and condition~2 of Theorem~\ref{thm:principal} holds if and only if $\Lambda > 0$ and $y^- \leq y^+$.
\end{lemma}

\subsection{Where $\varphi < \tau$, concavity}
\label{app:concave}

\begin{lemma}[Concavity and endpoints]\label{lem:concave}
$\varphi$ is concave on $[0, n]$, hence $\varphi(y) \geq \min\big(\varphi(0), \varphi(n)\big)$ for every $y \in [0, n]$. In particular, $I_- \neq \emptyset$ if and only if $\min\big(\varphi(0), \varphi(n)\big) < \tau$. Moreover,
\[
  \varphi(0) = \DP(C) + \frac{n\, b}{B(B + n)},
  \qquad
  \varphi(n) = \DP(C) + \frac{n\,(A - a)}{A(A + n)}.
\]
\end{lemma}

\begin{proof}
Writing $\frac{a + y}{A + y} = 1 - \frac{A - a}{A + y}$, we have $\varphi(y) = 1 - \frac{A - a}{A + y} - \frac{b}{B + n - y}$, hence
\[
  \varphi'(y) = \frac{A - a}{(A + y)^2} - \frac{b}{(B + n - y)^2},
  \qquad
  \varphi''(y) = -\frac{2(A - a)}{(A + y)^3} - \frac{2b}{(B + n - y)^3} \leq 0 ,
\]
since $A - a \geq 0$, $b \geq 0$ and both denominators are positive. Every $y \in [0, n]$ is $y = \lambda n$ with $\lambda \in [0, 1]$, and concavity gives $\varphi(y) \geq (1 - \lambda)\varphi(0) + \lambda\varphi(n) \geq \min\big(\varphi(0), \varphi(n)\big)$. So if both endpoints are at least $\tau$, then $\varphi \geq \tau$ everywhere and $I_- = \emptyset$; conversely, if one endpoint is below $\tau$, it belongs to $I_-$. For the endpoint values, using $\DP(C) = a/A - b/B$,
\[
  \varphi(0) - \DP(C) = \frac{b}{B} - \frac{b}{B + n} = \frac{n\, b}{B(B + n)},
  \qquad
  \varphi(n) - \DP(C) = \frac{a + n}{A + n} - \frac{a}{A} = \frac{n\,(A - a)}{A(A + n)} . \qedhere
\]
\end{proof}

In practice, condition~1 holds automatically once the \ac{GAS vector} is large compared with the attacker's budget.

\begin{lemma}[Condition~1 in practice]\label{lem:cond1}
If $\DP(C) < -\tau$ and $\max(A, B) \geq n\,\big(1/(2\tau) - 1\big)$, then condition~1 of Theorem~\ref{thm:principal} holds.
\end{lemma}

\begin{proof}
By Lemma~\ref{lem:concave}, condition~1 reads $\min\big(\varphi(0) - \DP(C),\ \varphi(n) - \DP(C)\big) < \tau - \DP(C)$. Since $b \leq B$ and $A - a \leq A$, the endpoint values of Lemma~\ref{lem:concave} give
\[
  \varphi(0) - \DP(C) = \frac{n\, b}{B(B + n)} \leq \frac{n}{B + n},
  \qquad
  \varphi(n) - \DP(C) = \frac{n\,(A - a)}{A(A + n)} \leq \frac{n}{A + n},
\]
so the left-hand side is at most $\frac{n}{\max(A, B) + n}$. This is at most $2\tau$ if and only if $\max(A, B) \geq n\,(1 - 2\tau)/(2\tau) = n\,\big(1/(2\tau) - 1\big)$, which holds by assumption. Finally, $\tau - \DP(C) > 2\tau$ since $\DP(C) < -\tau$.
\end{proof}

\subsection{Small jumps}
\label{app:jumps}

\begin{lemma}[Small jumps]\label{lem:jumps}
For every $y \in \{0, \dots, n - 1\}$,
\[
  |\varphi(y + 1) - \varphi(y)| \leq \frac{1}{\min(A, B) + 1}.
\]
\end{lemma}

\begin{proof}
We compute the variation of each term of $\varphi$ when $y$ increases by one. For the first term,
\[
  \frac{a + y + 1}{A + y + 1} - \frac{a + y}{A + y}
  = \frac{(a + y + 1)(A + y) - (a + y)(A + y + 1)}{(A + y)(A + y + 1)}
  = \frac{A - a}{(A + y)(A + y + 1)} =: u_y ,
\]
since the numerator simplifies to $(A + y) - (a + y) = A - a$. For the second term,
\[
  \frac{b}{B + n - y - 1} - \frac{b}{B + n - y}
  = \frac{b}{(B + n - y - 1)(B + n - y)} =: v_y .
\]
Hence $\varphi(y + 1) - \varphi(y) = u_y - v_y$ with $u_y, v_y \geq 0$, so that $|\varphi(y + 1) - \varphi(y)| \leq \max(u_y, v_y)$. The denominator of $u_y$ grows with $y$, so $u_y$ is largest at $y = 0$, where $u_0 = \frac{A - a}{A(A + 1)} \leq \frac{1}{A + 1}$ since $A - a \leq A$. The denominator of $v_y$ decreases with $y$, so $v_y$ is largest at $y = n - 1$, where $v_{n - 1} = \frac{b}{B(B + 1)} \leq \frac{1}{B + 1}$ since $b \leq B$.
\end{proof}

\subsection{Proof of Theorem~\ref{thm:principal}}
\label{app:proof:main}

\begin{proof}[Proof of Theorem~\ref{thm:principal}]
By the roadmap, we must show that $I_+ \cap I_- \neq \emptyset$ if and only if conditions~1 and~2 hold.

\emph{Necessity.} Let $y \in I_+ \cap I_-$. Then $I_- \neq \emptyset$, which is condition~1 by Lemma~\ref{lem:concave}, and $y \in I_+$ is condition~2. This direction does not use~\eqref{eq:hypM}.

\emph{Sufficiency.} $I_+$ is nonempty by condition~2, and $I_-$ is nonempty by condition~1 and Lemma~\ref{lem:concave}. Assume by contradiction that $I_+ \cap I_- = \emptyset$. Then $\varphi \geq \tau$ on $I_+$, and $I_+ \neq \{0, \dots, n\}$ since $I_-$ is nonempty. As $I_+$ is a set of consecutive integers (Lemma~\ref{lem:Iplus}), there exist two neighbors $y \in I_+$ and $y' \notin I_+$, $|y - y'| = 1$, with $\varphi(y) \geq \tau$ and $\varphi(y') \leq -\tau$. Hence
\[
  |\varphi(y) - \varphi(y')| \geq 2\tau > \frac{1}{\min(A, B) + 1}
\]
by~\eqref{eq:hypM}, which contradicts Lemma~\ref{lem:jumps}.
\end{proof}

\begin{corollary}[The attacker reaches the fair band]\label{cor:attack}
Under the assumptions of Theorem~\ref{thm:principal}, if conditions~1 and~2 hold, the attacker~\eqref{eq:attack} satisfies $|\DP(C + \delta^\star)| < \tau$.
\end{corollary}

\begin{proof}
By Theorem~\ref{thm:principal}, there exists $y \in \{0, \dots, n\}$ with $|\varphi(y)| < \tau$. The extreme-label move $\delta = (y, y, 0, n - y)$ belongs to $\mathcal{X}_{n}$ and satisfies $\DP(C + \delta) = \varphi(y)$. Since $\delta^\star$ minimizes $|\DP(C + \delta)|$ over $\mathcal{X}_{n}$, we get $|\DP(C + \delta^\star)| \leq |\varphi(y)| < \tau$.
\end{proof} 
\newpage
\section{Additional experimental details}
\label{app:exp_details}

This appendix complements the setup of Section~\ref{sec:exp_setup} with additional details: how we obtain the audited \acp{LLM}' predictions on each dataset (Appendices~\ref{app:prompts:civil} and~\ref{app:prompts:biasbios}), how we partition the inputs across auditors, and our implementation and hardware specifications.
In both datasets we use the audited \ac{LLM} as a zero-shot classifier. Every input is sent as a separate request, with no in-context examples and no fine-tuning.

\subsection{\civil dataset}
\label{app:prompts:civil}

The audited \ac{LLM} is \qwen, which we ask whether each comment is toxic. A positive prediction means the \ac{LLM} flags the comment as toxic. The \ac{LLM}'s answer could not be parsed for 4 of the 182,176 comments, and we drop these, leaving 182,172 comments.  Each comment is annotated with one of 16 identity groups. Figure~\ref{fig:prompt-civil} shows the prompt.

\begin{figure}[t]
\begin{promptbox}[title={Prompt: toxicity classification (\civil)}]
You are a content moderation system. Decide whether the following comment is toxic: rude, disrespectful or unreasonable language that is likely to make someone leave a discussion. Reply with JSON of the form {"items": [{"yes": Y, "no": N}]}, where Y is your confidence in percent that the comment is toxic, N your confidence that it is not, and Y + N = 100.
\end{promptbox}
\caption{Prompt used to classify each comment in \civil with \qwen.}
\label{fig:prompt-civil}
\end{figure}

\subsection{\biasbios dataset}
\label{app:prompts:biasbios}

\textbf{Data.} We use the LabHC release of \biasbios~\cite{de2019bias}, stacking the train, test, and dev splits in that order (257,478 + 99,069 + 39,642 biographies). We keep the same 25 occupations as \cite{hartmann2026audit} and drop \emph{interior designer}, \emph{personal trainer}, and \emph{rapper}. This leaves 391,895 biographies. Each biography has a binary gender label.

\textbf{LLM and prompt.} The audited \ac{LLM} is \llama, served with vLLM. Each biography is sent as a single user message formatted with the \ac{LLM}'s chat template (Figure~\ref{fig:prompt-biasbios}). We add no separate system prompt, so that the same message can be used with every \ac{LLM} we evaluate.

\textbf{Decoding.} We decode greedily (temperature 0, fixed seed) with at most 16 new tokens. We use vLLM's structured outputs to restrict the answer to exactly one of the 25 occupation names. Every biography therefore gets a valid label, and unlike in \civil, no input is dropped. The audit uses only the predicted label; we do not ask the \ac{LLM} for per-class scores.

\textbf{Audits.}
We group the 25 occupations into sectors following the U.S.\ Standard Occupational Classification.
An audit fixes a sector and a target occupation within it: its population consists of all biographies whose true occupation lies in that sector, and a prediction is positive if and only if the \ac{LLM} predicts the target occupation (Table~\ref{tab:6_cases}).
The audit thus measures whether, among people working in the same sector, the \ac{LLM} assigns the target occupation at different rates to men and women.

\subsection{Data partition}
For every combination of group and predicted label, the inputs are split across auditors with proportions drawn from a Dirichlet distribution of concentration $\alpha$.
We redraw the partition if any auditor holds fewer than $30$ inputs of either group.
The resulting local estimates are strongly heterogeneous: across cases, the 5th percentile of an auditor's local DP ranges from $-0.55$ to $-0.84$, and the 95th percentile from $+0.08$ to $+0.66$.

\begin{figure}[t]
\begin{promptbox}[title={Prompt: occupation classification (\biasbios)}]
You are an occupation classification system. You are given a short biography of a person. Answer with the person's occupation, chosen from this list: accountant, architect, attorney, chiropractor, comedian, composer, dentist, dietitian, dj, filmmaker, journalist, model, nurse, painter, paralegal, pastor, photographer, physician, poet, professor, psychologist, software\_engineer, surgeon, teacher, yoga\_teacher. Reply with the occupation name only, exactly as written in the list.\par
\medskip
Biography: \textit{\textrm{\{biography\}}}\par
\medskip
Occupation:
\end{promptbox}
\caption{Prompt used to classify each biography in \biasbios with \llama. \textit{\{biography\}} is replaced by the biography text. The answer is constrained to the 25 occupation names in the list.}
\label{fig:prompt-biasbios}
\end{figure}

\subsection{Implementation details and hardware}
We implement \sys and the audit simulation in Python using the \textsc{NumPy} and \textsc{Pandas} libraries.
The simulation is fully seeded.
Each run is determined by its configuration and seed, so every result can be reproduced exactly.
The audited \acp{LLM}' predictions are computed once, before the audits.
We serve the \acp{LLM} with \textsc{vLLM}~\cite{kwon2023efficient}.
We generate the predictions using NVIDIA A100 GPUs (one model replica per GPU) available on our compute cluster.
The audit simulations all run on CPUs.
 
\newpage

\newpage
\section{Additional experiments}
\label{app:experiments}

We complement \Cref{sec:experiments} with additional experiments that answer the following questions:
\begin{itemize}
    \item \textbf{Generalization (RQ3).} How accurately do \sys and baselines preserve the audit's verdict and its DP estimate on the \civil dataset as the proportion of adversaries grows (\Cref{app:exp_civil})?
    \item \textbf{Query sampling (RQ4).} How does querying with replacement, instead of without replacement, affect the accuracy of an audit with honest auditors only, and how does this depend on data heterogeneity (\Cref{app:exp_sample})?
    \item \textbf{Baseline configuration (RQ5).} Do the robust aggregation baselines become robust when they aggregate each auditor's local DP instead of per-group positive rates, or when we change the trimming fraction $\beta$ of the trimmed mean (\Cref{app:detailed_baselines})?
    \item \textbf{Query budget (RQ6)}. How does the per-round query budget, which sets the size of each update that \sys scores, affect the accuracy and robustness of \sys{} and the baselines (\Cref{app:budget})? 
\end{itemize}

\begin{table*}[t]
\centering
\small
\caption{The \acf{ASR} and mean absolute \ac{DP} estimation error on the \civil dataset, over every run with $1$--$19$
adversaries ($5$--$49\%$ of all auditors; $380$ runs per case; as Table~\ref{tab:rq1-bios-table}). Bold: the lowest in
each row.}
\label{tab:rq1-cc-table}
\setlength{\tabcolsep}{5pt}

\begin{tabular}{lrrrrrrrr}
\toprule
 & \multicolumn{4}{c}{Attack success rate (\%)} & \multicolumn{4}{c}{Average DP estimation error} \\
\cmidrule(lr){2-5}\cmidrule(lr){6-9}
\ac{LLM} fairness & No-defense & Median & Trimmed mean & \sys{} & No-defense & Median & Trimmed mean & \sys{} \\
\midrule
Very unfair & 95 & 42 & 63 & \textbf{0} & 0.221 & 0.226 & 0.221 & \textbf{0.050} \\
Moderate & 99 & 33 & 35 & \textbf{0} & 0.186 & 0.235 & 0.219 & \textbf{0.022} \\
Near-fair & 97 & \textbf{68} & 80 & 78 & 0.061 & 0.081 & 0.073 & \textbf{0.028} \\
\midrule
All three cases & 97 & 48 & 59 & \textbf{26} & 0.156 & 0.180 & 0.171 & \textbf{0.034} \\
\bottomrule
\end{tabular}

\end{table*} 
\subsection{RQ3: Effectiveness of \sys and baselines on \civil Dataset}
\label{app:exp_civil}

\begin{figure*}[t]
  \centering
    \includegraphics[width=.85\textwidth]{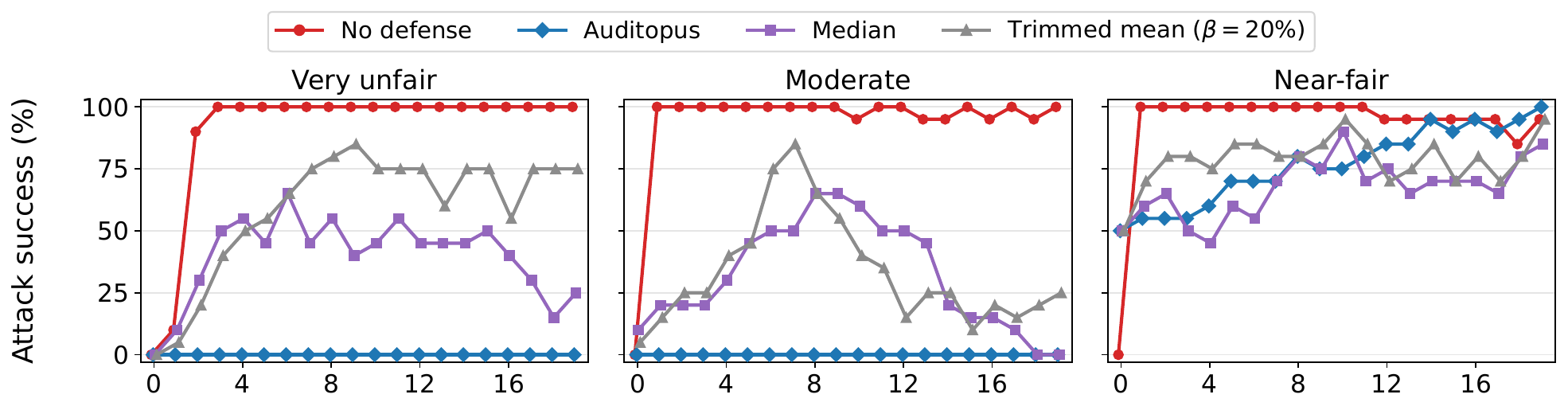} 
    \vspace{0.5em}
    \includegraphics[width=.85\textwidth]{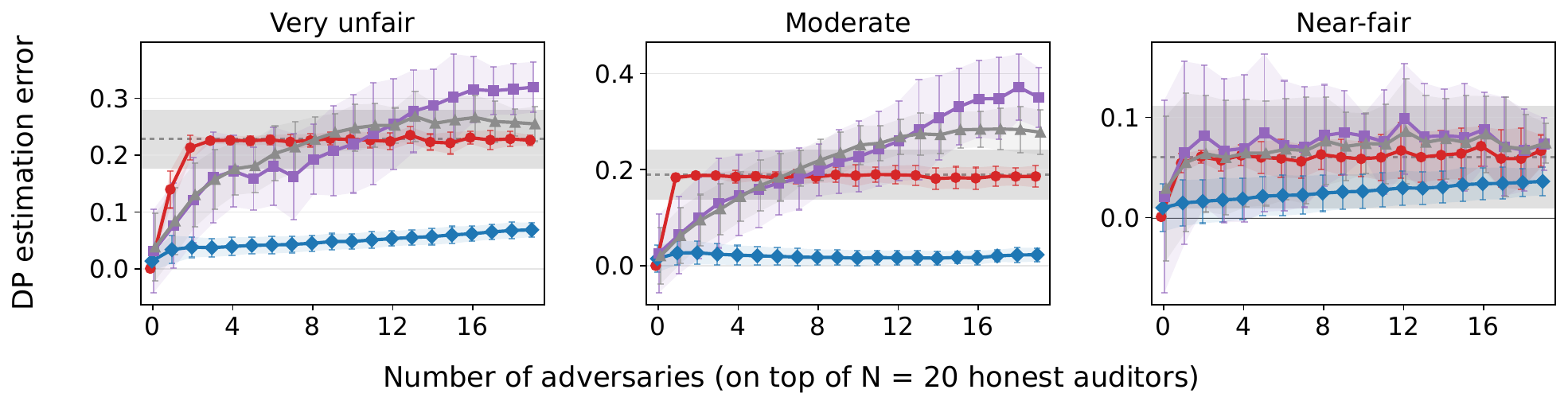}
    \caption{The attack success rate (top row) and \ac{DP} estimation error (in the bottom row) for \civil, while varying the number of adversaries $K$ from $0$ to $19$. We consider a very unfair (left column), moderate (middle column) and near-fair \ac{LLM} (right column). The dashed horizontal line in \Cref{fig:rq1-civil} (in the bottom row) corresponds to a \ac{DP} value of 0.}
  
  \label{fig:rq1-civil}
\end{figure*}

We repeat the effectiveness experiment of \Cref{sec:eval_rq1} on a second model and dataset: \qwen as a zero-shot toxicity classifier on \civil, with the same protocol, attack, baselines, and metrics. Each case compares two identity groups of the same category: its population consists of all comments about either group, and a positive prediction flags a comment as toxic. The three cases are Christian vs.\ Muslim (very unfair, true \ac{DP} $-0.228$), Christian vs.\ Jewish (moderate, $-0.190$), and Black vs.\ White (near-fair, $-0.060$), the last only $0.010$ outside the fair band (\Cref{tab:6_cases}). The honest auditors' data is as heterogeneous as on \biasbios, and \Cref{fig:rq1-civil,tab:rq1-cc-table} show the same picture. Across all runs with at least one adversary, the attack succeeds in $97\%$ of the audits without a defense, in $48\%$ and $59\%$ with the median and the trimmed mean, and in $26\%$ with \sys, all on the near-fair case. \sys lowers the mean \ac{DP} estimation error from $0.156$ without a defense to $0.034$ ($-78\%$), whereas the median ($0.180$) and the trimmed mean ($0.171$) are even less accurate than no defense at all. Without a defense, a single adversary fairwashes every run on the moderate and near-fair cases. \sys, in contrast, is never fooled on the very unfair and moderate cases, but it fails on the near-fair one. Already without adversaries it declares the model fair in $50\%$ of the runs: down-weighting honest auditors whose updates deviate from their own history by chance leaves a mean error of $+0.010$, exactly the distance to the band. The median and the trimmed mean do
the same ($50\%$ at $K = 0$). Under attack, \sys is fooled in $55$--$80\%$ of the runs up to $8$ adversaries and in every run at $K = 19$. Its estimate nevertheless stays closest to the truth of all methods (mean error $+0.015$ to
$+0.036$, against $+0.061$ without a defense; Table~\ref{tab:rq1-cc-table}), but an error of a few hundredths suffices to cross a band a hundredth away. A disparity this close to the threshold is at the limit of what an audit on heterogeneous data can resolve. Robust aggregation is pushed past the band. All three \civil cases have mid-range positive rates, so the median and the trimmed mean behave as on the moderate
\biasbios case: they are fooled in $33$--$68\%$ (median) and $35$--$80\%$ (trimmed mean) of the attacked runs, and they
do not stop at the band. At $K = 19$ the median's mean error exceeds the true disparity on both the very unfair and the moderate cases, \ie the estimate is pushed past $0$.

\subsection{RQ4: Sampling without replacement}
\label{app:exp_sample}
Throughout the paper, each auditor samples its query set without replacement. By the end of the audit it has queried every one of its rows once. Here we justify this choice by comparing it with querying with replacement, where each query is drawn uniformly from the auditor's query set. Both modes use the same data splits, graphs, seeds and number of rounds as RQ1, and we compare with the true DP.

\begin{figure*}[t]
  \centering
  \subfloat[\civil]{%
    \begin{minipage}{.85\textwidth}
      \includegraphics[width=\textwidth]{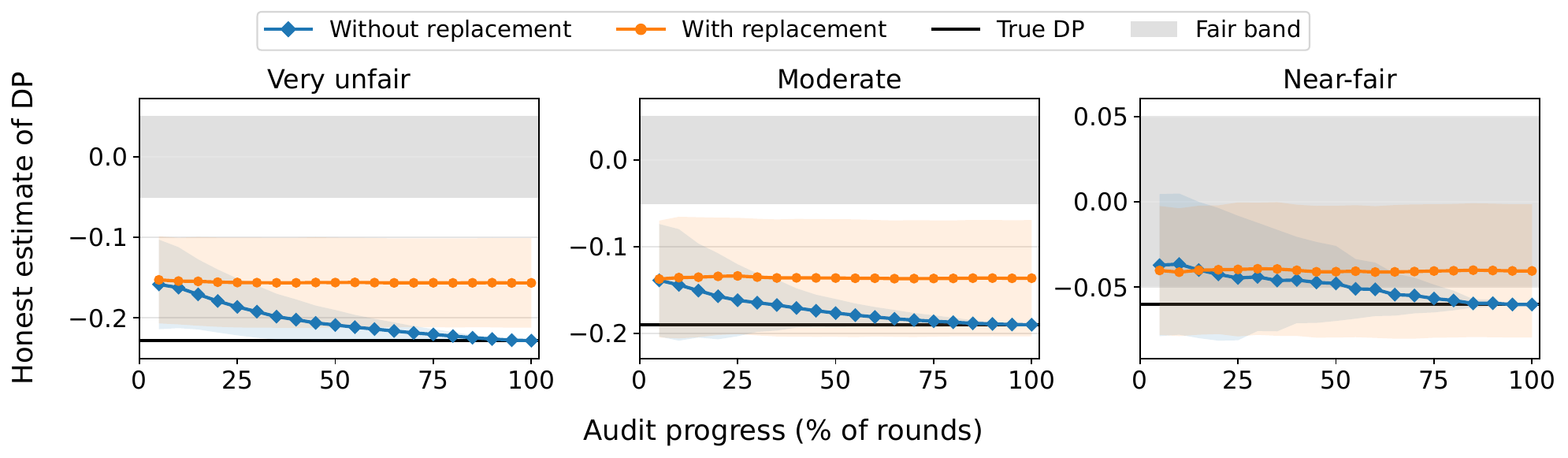}
      \includegraphics[width=\textwidth]{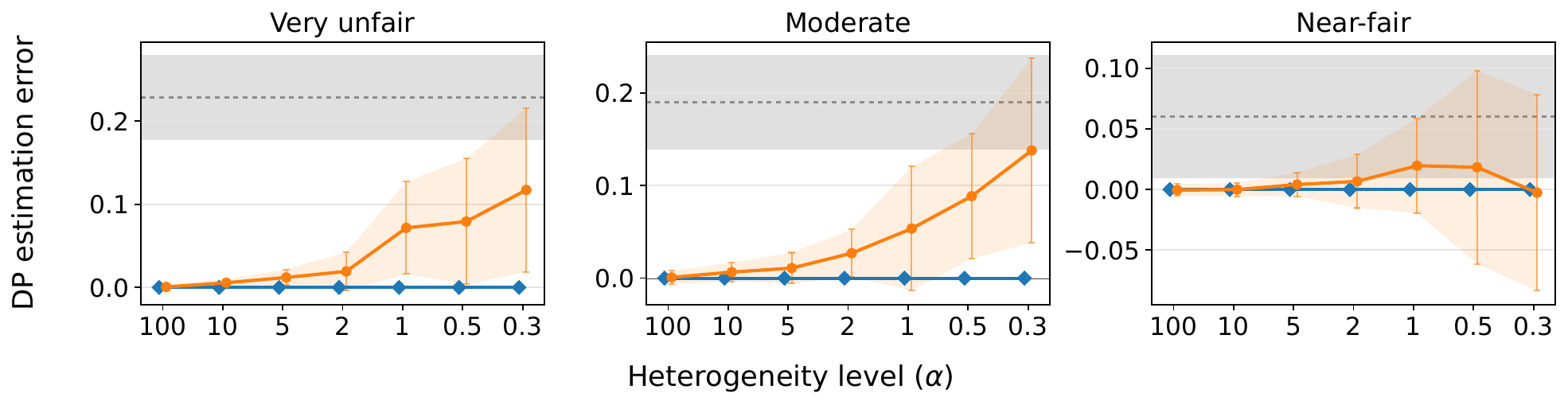}\\     
    \end{minipage}
  }\\[0.3em]
  \subfloat[\biasbios]{%
    \begin{minipage}{.85\textwidth}
      \includegraphics[width=\textwidth]{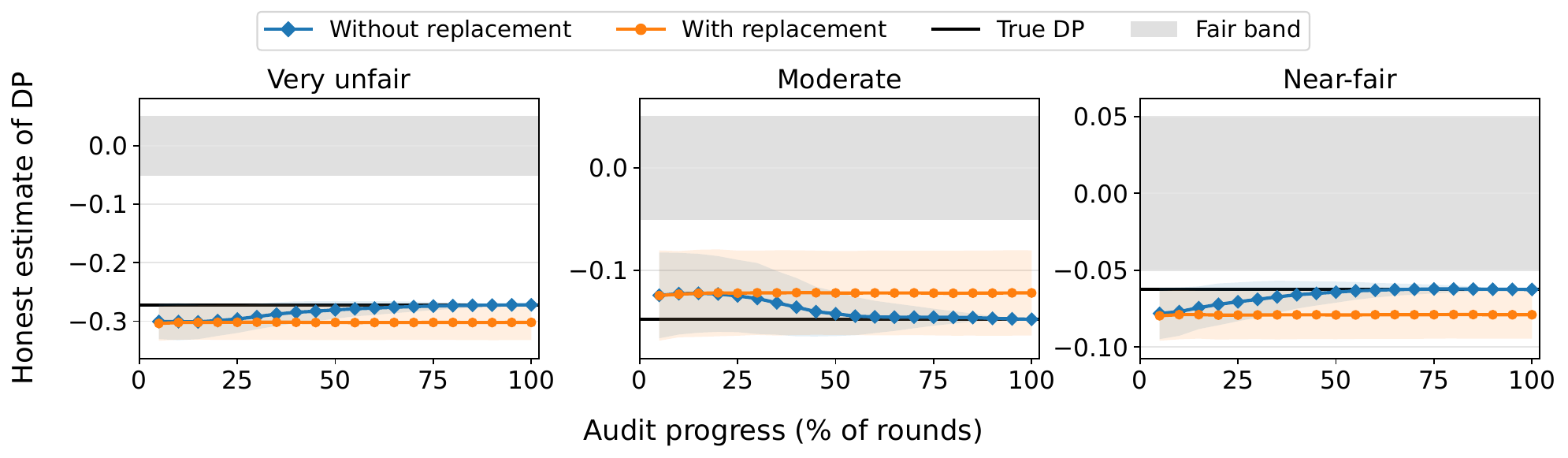}
      \includegraphics[width=\textwidth]{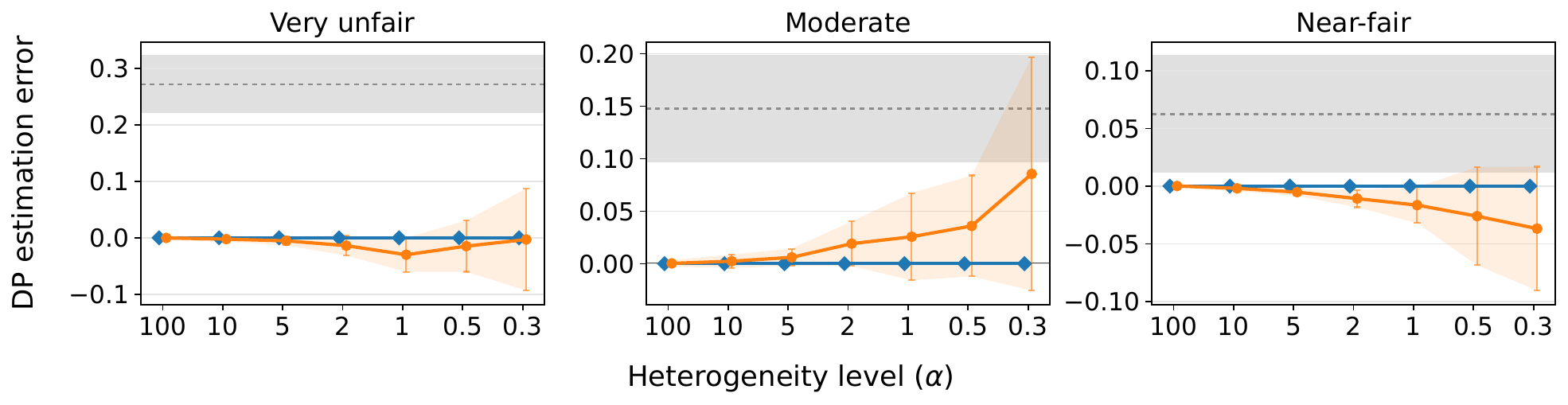}
    \end{minipage}
  }
  \caption{Honest auditors only ($N = 20$, no adversary), querying without and with replacement (mean $\pm$ SD over
  $20$ runs; top rows: \civil, bottom rows: \biasbios). Without replacement the estimate converges to the true
  DP; with replacement it stays where it started, off the true DP by an amount that grows with heterogeneity.} 
  \label{fig:app-honest-replacement}
\end{figure*}

With replacement, even honest auditors never reach the true DP. Figure~\ref{fig:app-honest-replacement} follows the honest auditors' aggregated estimate through the audit, with no adversary. In both modes the first rounds give the same estimate, off the true DP. Every auditor contributes $300$ queries per round, so a small query set counts as much as a large one. Without replacement, the small query sets run out as the audit proceeds while the large ones keep contributing, and at the end the estimate equals the true DP in every run. With replacement, every auditor keeps drawing $300$ queries per round from its query set, so the estimate stays where it started. On Christian vs.\ Muslim, for instance, it goes from $-0.151$ after the
first round to $-0.157$ at the end, against a true DP of $-0.228$, a third of the disparity lost. On the near-fair Black vs.\ White case, the offset even puts the estimate inside the fair band. The bias grows with heterogeneity. With near-homogeneous data ($\alpha = 100$), all query sets look alike and weighting them equally changes nothing. The error is at most $0.001$. As the query sets become more heterogeneous it grows, \eg on Christian vs.\ Muslim from $+0.012$ at $\alpha = 5$ to $+0.072$ at $\alpha = 1$ and $+0.117$ at $\alpha = 0.3$. Heterogeneous query sets are the setting of this paper, so querying with replacement would bias every audit, independently of any adversary.

\subsection{RQ5: Robust-aggregation baselines}
\label{app:detailed_baselines}
In the main text, the median and the trimmed mean aggregate each group's positive rate over the sources of an auditor's table ($a/A$ and $b/B$ separately), and the trimmed mean cuts $\beta = 20\%$ of the sources from each end. Here we test whether a different choice rescues them: aggregating each source's local DP $a/A - b/B$ instead, and other trimming fractions $\beta \in \{10, 20, 30, 40\}\%$. Everything else is as in RQ1.

\textbf{Aggregating local DPs.}
\Cref{fig:app-local} applies the median and the trimmed mean to each source's local \ac{DP} instead of its per-group positive rates.
Heterogeneity biases both aggregators even without adversaries: the honest auditors' local \acp{DP} are spread widely around the true \ac{DP} (\Cref{fig:collaborative-auditing}), so their median is not the pooled \ac{DP}. With no adversary, it already declares the model fair in $60\%$ of the audits on \textit{teacher} and $35\%$ on \textit{physician}.

\textbf{Changing the trimming fraction.}
Trimming more makes the trimmed mean closer to the median (Figure~\ref{fig:app-trims}). No fraction comes close to \sys. Even choosing the best $\beta$ for each case separately, which an auditor could not do without knowing the true DP, the trimmed mean is fooled in $47\%$ of the runs on Christian vs.\ Muslim, $32\%$ on Christian vs.\ Jewish, $48\%$ on physician and $25\%$ on teacher, where \sys is fooled in at most $9\%$. Only on \textit{nurse} (both $0\%$) and on the near-fair Black vs.\ White ($73\%$ against $78\%$) does it match \sys. Larger fractions also inherit the median's other failure: on nurse at $K = 10$, the estimate is pushed away from the band, overstating the disparity; and without adversaries, the false fair verdicts on teacher rise from $20\%$ to $35\%$ as $\beta$ grows.

Neither choice makes robust aggregation robust. Aggregating local DPs trades being pushed away from the band for being pushed into it, and no trimming fraction closes the gap to \sys. Both fail for the same reason as in the main text. On heterogeneous data the honest auditors' values are not concentrated, which is the assumption robust aggregation relies on.

\begin{figure*}[t]
  \centering
  \subfloat[\civil\label{fig:app-local-success}]{%
    \begin{minipage}{.85\textwidth}
      \includegraphics[width=\textwidth]{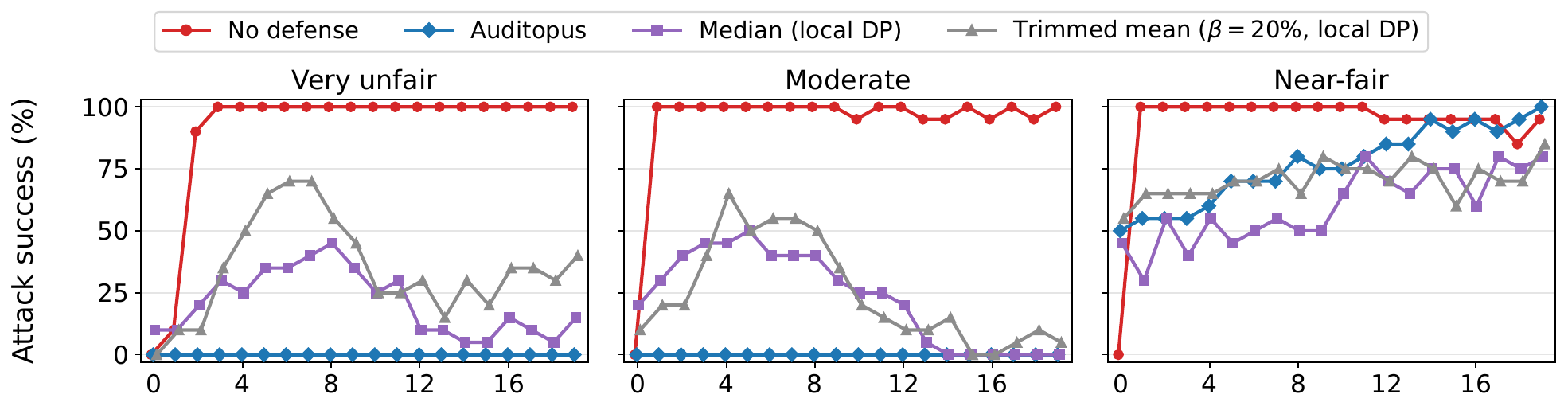}\\
      \includegraphics[width=\textwidth]{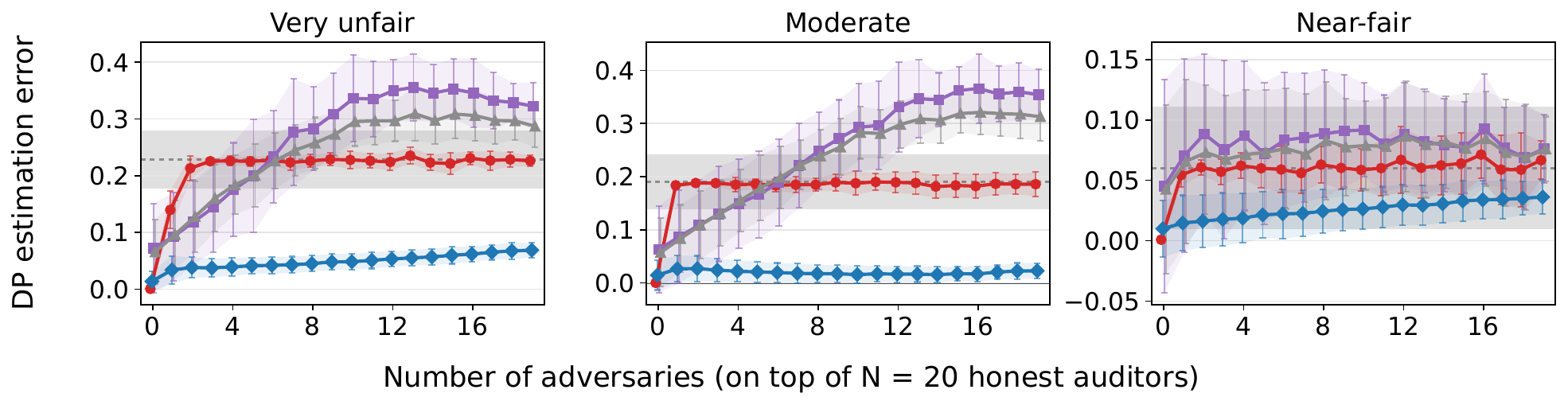}l
      \end{minipage}
  }\\[0.3em]
  \subfloat[\biasbios\label{fig:app-local-signed}]{%
    \begin{minipage}{.85\textwidth}
      \includegraphics[width=\textwidth]{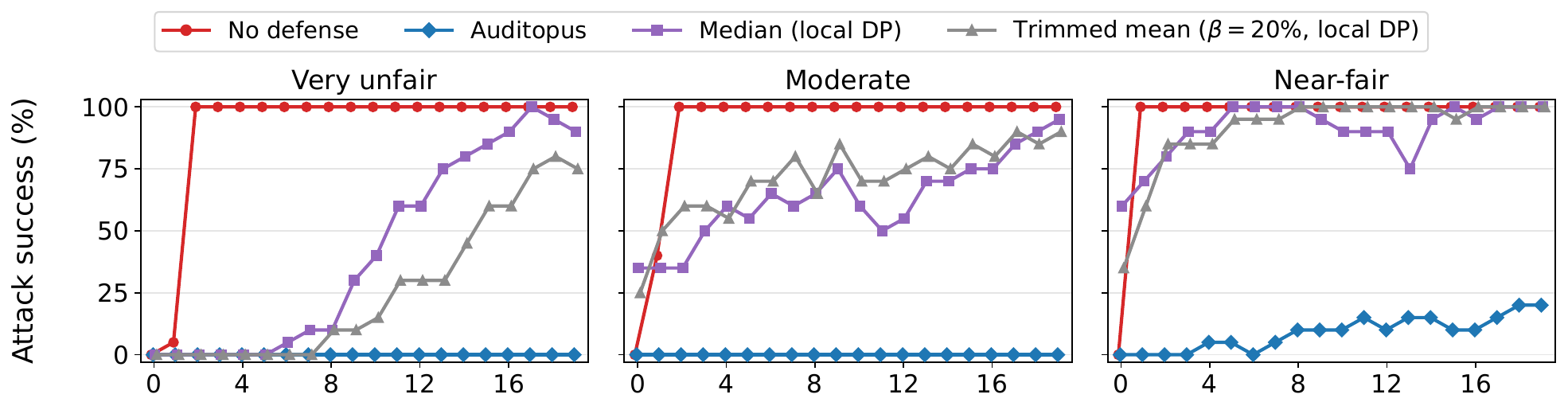}\\
      \includegraphics[width=\textwidth]{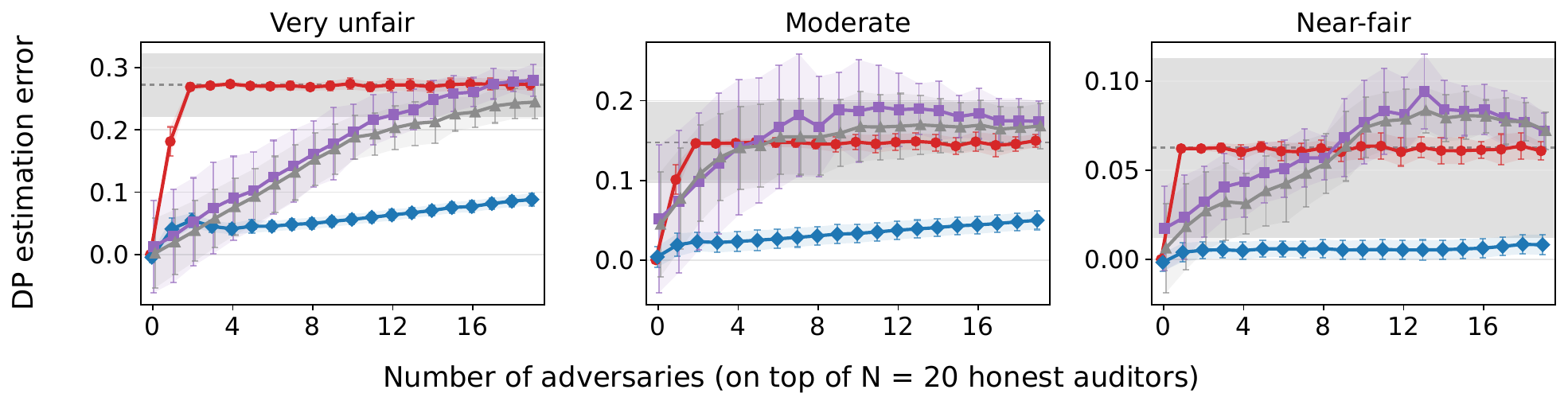}
    \end{minipage}
  }
  \caption{The median and the $20\%$ trimmed mean applied to each source's local DP estimate, against no defense and \sys{}, in
  the setting of RQ1 (top rows: \civil; bottom rows: \biasbios).}
  \label{fig:app-local}
\end{figure*}

\begin{figure*}[t]
  \centering
  \subfloat[\civil \label{fig:app-trims-success}]{%
    \begin{minipage}{.85\textwidth}
      \includegraphics[width=\textwidth]{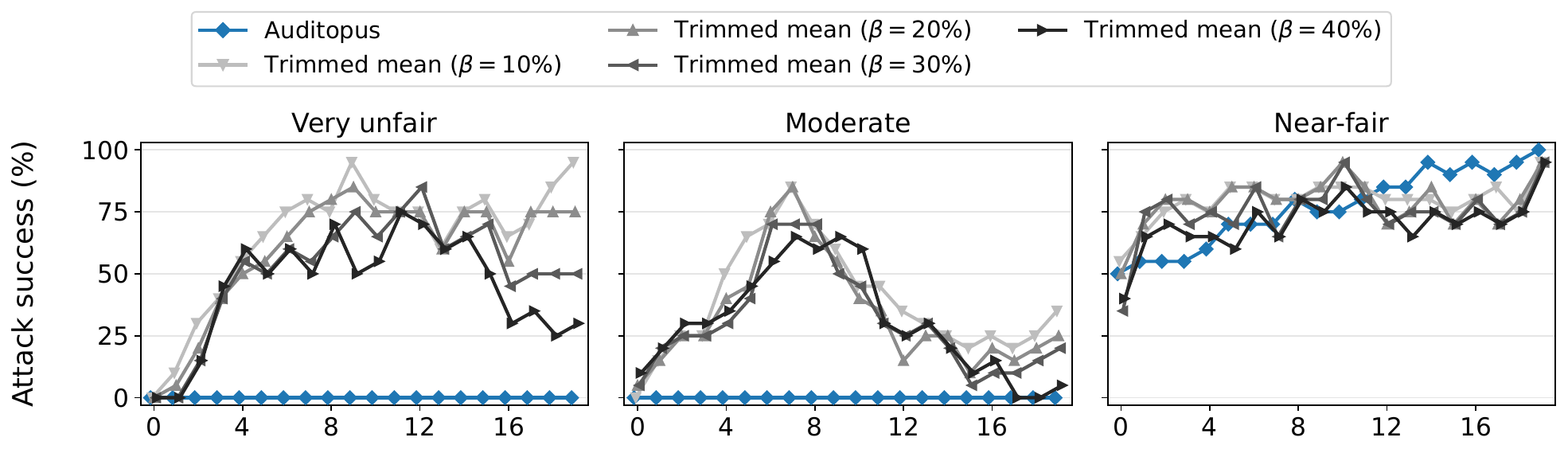}\\
      \includegraphics[width=\textwidth]{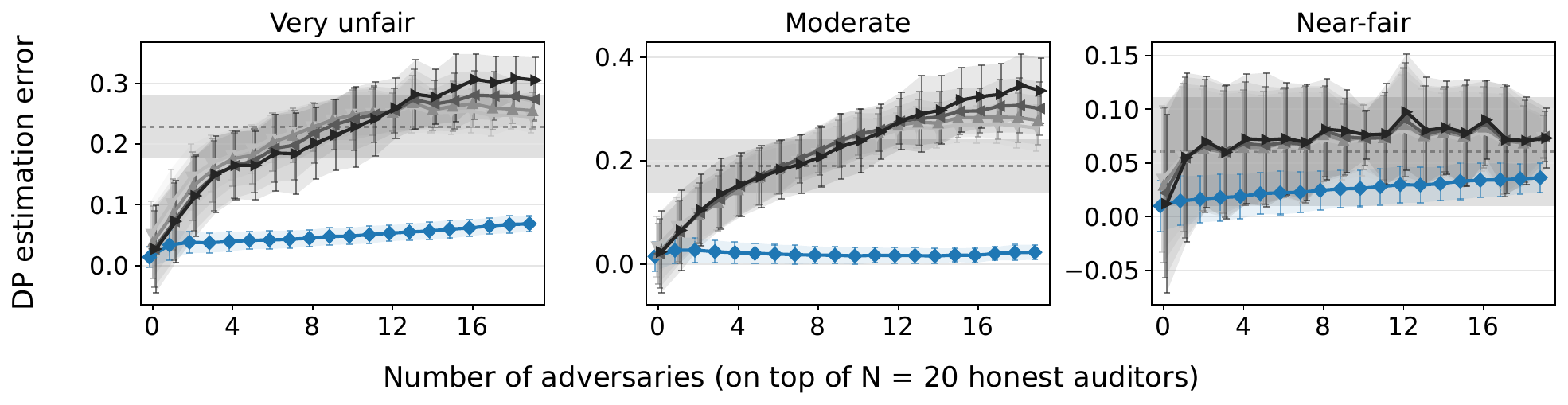}
    \end{minipage}
  }\\[0.3em]
  \subfloat[\biasbios\label{fig:app-trims-signed}]{%
    \begin{minipage}{.85\textwidth}
      \includegraphics[width=\textwidth]{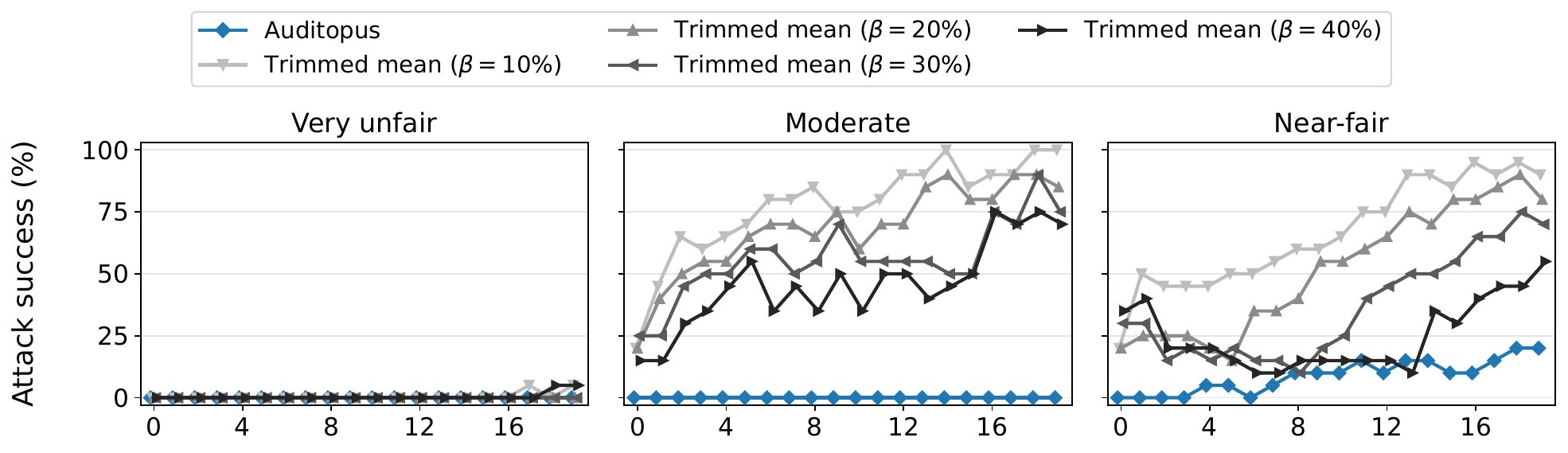}\\
      \includegraphics[width=\textwidth]{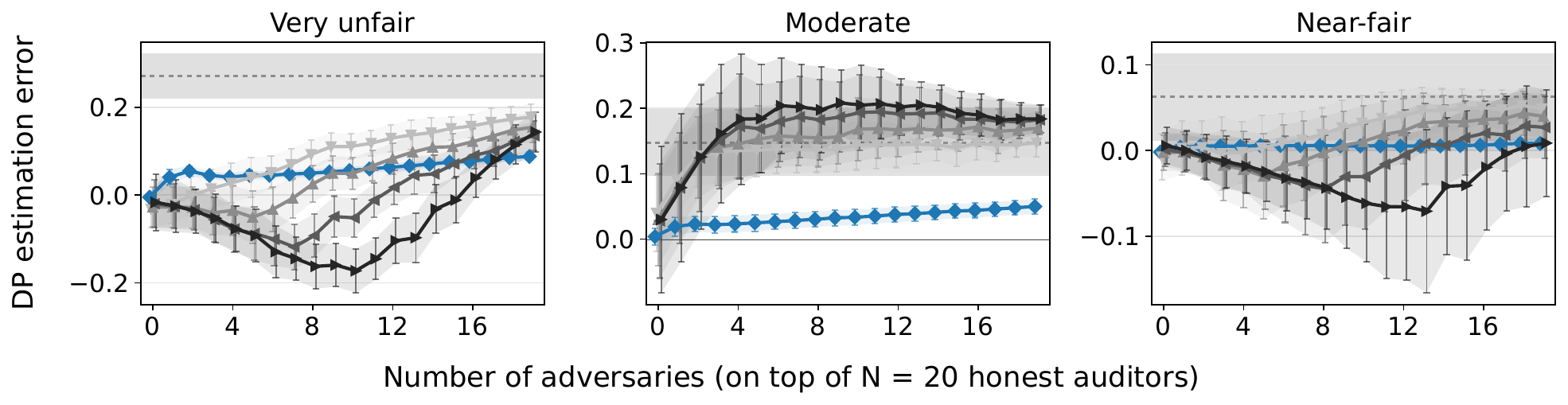}
    \end{minipage}
  }
  \caption{The trimmed mean of the sources' group rates at trimming fractions $\beta = 10$, $20$, $30$ and $40\%$,
  against \sys{}, in the setting of RQ1.}
  \label{fig:app-trims}
\end{figure*}

\subsection{RQ6: Varying the query budget}
\label{app:budget}

In the main text every auditor issues $n = 300$ queries per round. Here we vary $n$ from $100$ to $3{,}000$, with
$K = 6$ adversaries ($30\%$ of the $N = 20$ honest auditors), who fabricate $n$ queries per round as well; everything
else is as in RQ1. Since auditors query without replacement, every honest auditor still queries its whole query set,
whatever the budget. The budget only changes how the audit is split into rounds. The number of rounds is
$\lceil \max_i |\mathcal{D}_i| / n \rceil$, from $100$--$172$ at $n = 100$ to only $1$--$6$ at $n = 3{,}000$, and with it the size of each update that \sys scores. 
$T = 150$ cap defined in \Cref{sec:exp_setup} is for $n = 300$.  

\sys{} is unaffected as long as the audit has several rounds. Figure~\ref{fig:app-budget} shows attack success on both datasets. On the very unfair and moderate cases of both datasets, \sys{} is never fooled at any budget up to $3{,}000$. On the near-fair case in \biasbios it is fooled in at most $15\%$ of the runs up to $n = 2{,}000$, and in $40\%$ at $3{,}000$, where the audit is down to about six rounds. With so few updates per source, \sys{} has little history to judge each update against. On Black vs.\ White, which lies only $0.010$ outside the band, \sys{} is fooled in $65$--$95\%$ of the runs at every budget, as in RQ3.

The undefended audit is fooled less often at the largest budgets. Without a defense, the attack succeeds in at least $80\%$ of the runs of every case up to $n = 2{,}000$ ($70\%$ on Christian vs.\ Jewish), but less often at $3{,}000$ (\eg $60\%$ on physician and $45\%$ on Christian vs.\ Jewish). This is not a defense but a shorter audit: with only a few rounds, gossip cannot deliver the adversaries' latest counts to every honest auditor before the audit ends, so the adversaries' fabricated counts reach fewer auditors, whose estimates stay further from $0$.
\sys{}'s advantage does not depend on the query budget as long as the audit spans several rounds. When a large budget
shrinks the audit to a handful of rounds, \sys{} has little history per source to compare against, the same limitation
that makes a one-shot audit undefendable (Figure~\ref{fig:one_shot_audit}).

\begin{figure*}[t]
  \centering
  \subfloat[\civil \label{fig:app-budget-civil}]{%
      \includegraphics[width=.85\textwidth]{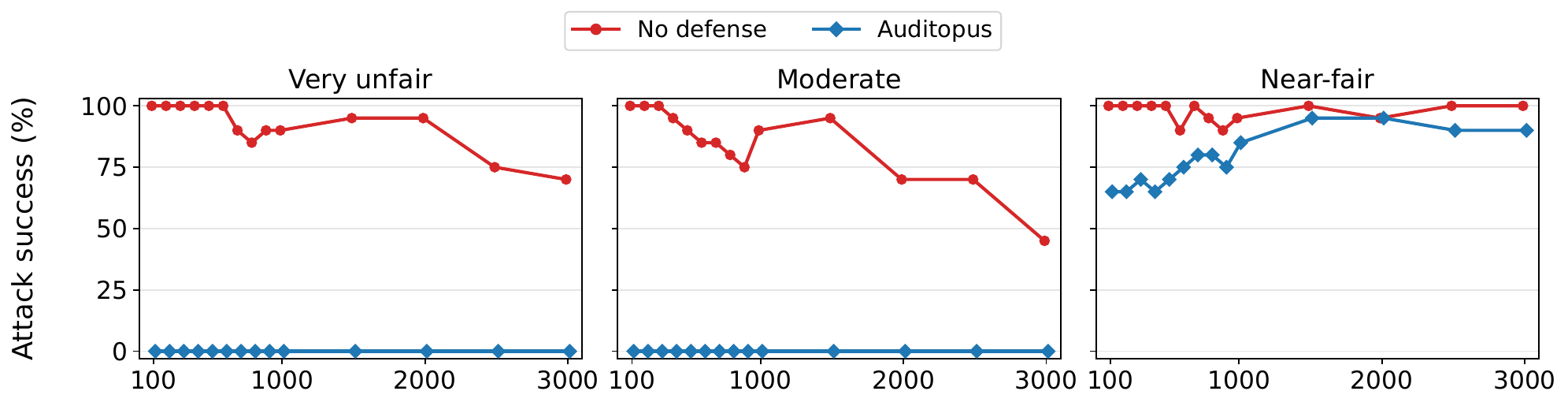} 
  }\\[0.3em]
  \subfloat[\biasbios\label{fig:app-budget-bios}]{%
    \includegraphics[width=.85\textwidth]{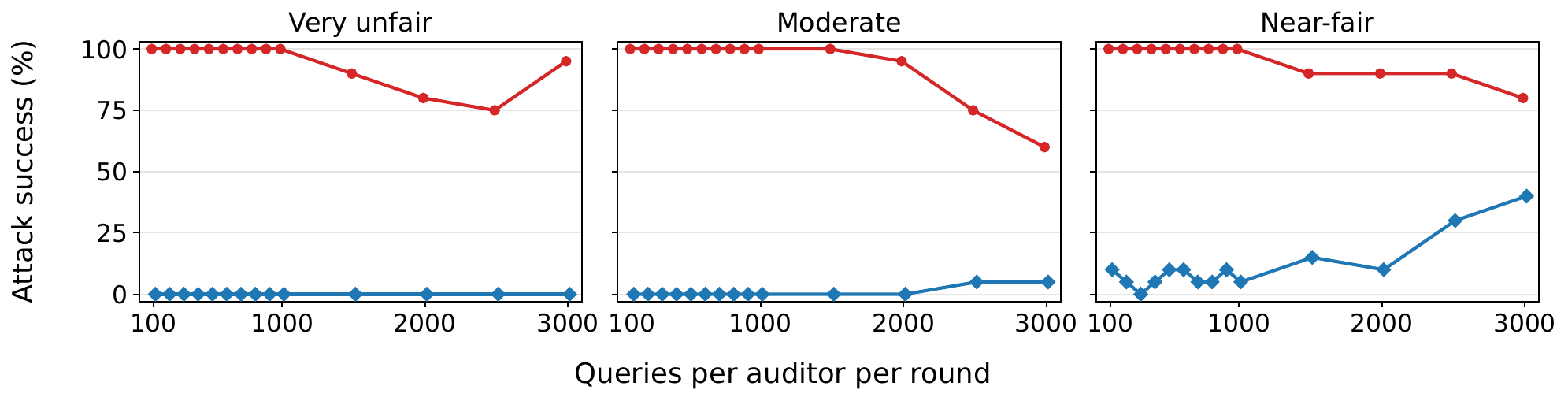}
  }
  \caption{Attack success as the number of queries per auditor per round grows from $100$ to $3{,}000$, with $K = 6$
  adversaries next to $N = 20$ honest auditors ($20$ runs per point); \civil (top) and \biasbios (bottom).}
  \label{fig:app-budget}
\end{figure*}
 
\end{document}